\documentclass[12pt]{article}

\usepackage[margin=1.25in]{geometry} 

\usepackage{amsmath}
\usepackage{amssymb}
\usepackage{amsthm}
\usepackage{amsfonts}
\usepackage{mathrsfs}
\usepackage{url}
\usepackage{graphicx}
\usepackage{xcolor}
\usepackage[round]{natbib}
\usepackage{algorithm}
\usepackage{hyperref}
\usepackage{algpseudocode}

\usepackage{authblk}

\newtheorem{theorem}{Theorem}

\newtheorem{assumption}{Assumption}
\newtheorem{definition}{Definition}
\newtheorem{example}{Example}

\title{Scalable posterior uncertainty for density-based clustering\thanks{The authors are grateful to Vansh Bansal, Daniel W. Barreto, Yunshan Duan, Bernardo Flores, Nevena Gligi\'c, and Alessandro Rinaldo for helpful discussions, and especially to Shuai Guo and Peter M\"uller for their input on the single-cell RNA sequencing application.}}

\author[1]{Nicola Bariletto}
\author[2]{Stephen G. Walker}

\affil[1]{The University of Texas at Austin}
\affil[2]{University of Kent}

\date{\today}

\begin{document}

\maketitle

\begin{abstract}
Density-based clustering is a well-established framework that defines clusters through topological or geometric features of the data-generating density, yielding flexible procedures that have proven effective on large and complex datasets. We address the problem of obtaining scalable Bayesian-style uncertainty quantification by combining density-based clustering with martingale posterior distributions. Density uncertainty is first quantified through a fully parallelized predictive resampling scheme driven by model score functions, naturally accommodating modern differentiable density estimators such as normalizing flows. Posterior samples of the clustering are then obtained by applying density-based clustering to each random density sampled from the martingale posterior, enabling inference on any clustering-related quantity. Casting the estimand as a density functional also enables a rigorous frequentist convergence analysis. We illustrate the method through applications to image and single-cell RNA sequencing data, where it recovers meaningful clustering structures, including the assignment of previously unseen data and biologically interpretable cell grouping, while running in minutes on a single GPU.
\end{abstract}

\section{Introduction}\label{sec-intro}

Bayesian computation has traditionally relied on Markov chain Monte Carlo (MCMC) methods, which require specifying both a prior and a likelihood model and can be prohibitively slow for tasks involving large-scale datasets. The recently introduced \emph{martingale posterior distribution} \citep[MPD,][]{fong2023martingale} framework offers an efficient alternative: posterior samples are obtained in parallel through a predictive resampling scheme, requiring neither a prior-likelihood model nor MCMC. In this article, we leverage MPDs to perform scalable Bayesian-style uncertainty quantification in large-scale clustering problems. In particular, we focus on \emph{density-based clustering} \citep[DBC,][]{kriegel2011density}, a paradigm valued for its robustness to irregularly shaped clusters and its compatibility with any underlying density estimator, including modern neural-network-based ones that adapt well to large and complex datasets. We propagate martingale posterior uncertainty on the density, obtained via a fully parallelizable, score-based predictive resampling scheme, to uncertainty on the clustering structure as implied by a DBC procedure of choice. The resulting procedure scales to large datasets, runs efficiently on modern GPU hardware, admits a tractable frequentist asymptotic analysis, and, as we illustrate on image and single-cell RNA sequencing data, recovers domain-relevant structure with associated uncertainty in minutes.

Clustering aims to partition data, and potentially the entire sample space, into subsets that are internally similar yet mutually distinct. Bayesian methods are well suited to quantifying the resulting uncertainty. The classical Bayesian approach, known as \emph{model-based clustering}, specifies a mixture model and identifies clusters with groups of observations assigned to the same component \citep{mclachlan2000finite, fraley2002model}; it aligns naturally with MCMC algorithms exploiting data augmentation via cluster labels \citep{escobar1995bayesian, mueller1996,neal2000} and admits flexible nonparametric formulations \citep{lo1984,lijoi2005hierarchical,deblasi2013gibbs,barrios2013modeling}. Uncertainty quantification within this framework is well developed and has recently been refined through repulsive mixtures \citep{xu2016bayesian, xie2020bayesian, petralia2012repulsive, beraha2025bayesian, cremaschi2025repulsion, song2025repulsive}, variable selection \citep{tadesse2005bayesian,chandra2023escaping}, scalable algorithms \citep{ni2020scalable}, and techniques for post-processing of partition-valued posteriors \citep{wade2018clustering, dahl2022search, bariletto2025conformalized}.

DBC, which has its roots in the computer science and machine learning communities \citep{chacon2015population, kriegel2011density}, adopts a different conceptualization of clusters, by viewing them as topological or geometric features of an estimated density---e.g., the path-connected components of its upper-level sets---rather than as latent mixture components. This formulation is robust to irregularly shaped clusters, imposes essentially no constraint on how the density is estimated, and typically partitions the entire sample space, so that unseen points can be clustered at little additional cost; we discuss these properties, along with formal definitions, in Section~\ref{sec:methodology}.

Despite these features, scalable Bayesian uncertainty quantification has not been developed for DBC. The closest existing work is \citet{buch2024bayesian}, who apply DBC to MCMC samples from traditional mixture models; this performs well in low-dimensional settings but inherits the computational and statistical limitations of MCMC-based mixtures in larger problems. We argue that a computationally tractable framework for Bayesian uncertainty in DBC is particularly timely, given the recent progress in deep learning methods for handling data of significant scale and complexity. This progress is often theoretically motivated by the \emph{manifold hypothesis} \citep{fefferman2016testing,bengio2013representation,pope2021intrinsic}, stating that while many modern datasets are high-dimensional in their ambient space representation, they in fact concentrate near much lower-dimensional structures; this pattern has also attracted recent interest in the statistical community \citep{whiteley2026statistical}.

For the DBC applications we are interested in, the manifold hypothesis implies that, while traditional density estimation in $\mathbb R^p$ suffers from the curse of dimensionality even when the data $X_{1:n}:=(X_1,\ldots,X_n)$ are intrinsically low-dimensional, neural density estimators will remain statistically efficient (and computationally tractable) because they are designed to discover and exploit such low-dimensional representations. Thus, combining DBC, which relies on density estimation as a fundamental building block, with neural density estimators and MPDs is a natural step towards scalable solutions for clustering applications.

The basic pipeline we propose, therefore, is to first characterize density uncertainty through the MPD framework, which is operationalized via a fully parallelizable, score-based predictive resampling scheme; then, to propagate MPD uncertainty to DBC by applying the clustering procedure of choice to each posterior sample of the density. This yields posterior samples of the induced clustering structure, from which any clustering-related quantity, such as the number of clusters, co-clustering probabilities, or credible sets over partitions, can be derived.

The remainder of the article is organized as follows. Section~\ref{sec:methodology} presents the methodology in detail, Section~\ref{sec:consistency} establishes frequentist convergence guarantees, and Section~\ref{sec:applications} illustrates the method on two real-world applications involving digit images and single-cell RNA sequencing data; Section~\ref{sec:discussion} concludes. Proofs and additional experimental details are collected in the Supplementary Material at the end of the document. The code to reproduce the numerical experiments is available at \url{https://github.com/nbariletto/MPD-DBC}.

\section{Methodology}\label{sec:methodology}

As laid out in the previous section, we propose a methodology that can be understood through a two-step procedure. Given a sample of observed data, a parameter-differentiable density model, and a DBC procedure of interest, the steps involve:
\begin{enumerate}
    \item \emph{Density uncertainty}: train a baseline density estimator from the selected family, e.g., by maximum likelihood; then obtain martingale posterior samples of the density in parallel, using score-based updates of the baseline estimate;\label{step1}

    \item \emph{Clustering uncertainty}: apply the chosen DBC procedure to each sampled density from Step~\ref{step1}, yielding martingale posterior samples of the clustering structure.\label{step2}
\end{enumerate}

Step~\ref{step1} does not strictly require score-based MPDs; other valid updating rules from the MPD literature (e.g., based on copulas) may be used. In fact, our asymptotic analysis of clustering consistency in Section~\ref{sec:consistency} is agnostic to the type of MPD, and the DBC procedure in Step~\ref{step2} may be applied to any collection of sampled densities. We nonetheless focus on score-based MPDs, and will provide specific theoretical guarantees for them, because this class admits instantiations based on flexible deep learning architectures that are essential for good performance in large-scale settings.

Step~\ref{step2} is equally flexible: a variety of off-the-shelf DBC algorithms can be applied to the sampled densities. In this section and in the theoretical analysis that follows, we focus on the prototypical DBC pipeline already described, which essentially corresponds to DBSCAN \citep{ester1996density}. Theory in this setting is both tractable and informative about extensions to more complex algorithms. In the real-data applications of Section~\ref{sec:applications}, however, we use ToMATo \citep{chazal2013persistence}, a more robust instantiation of DBC better suited to the complex data we analyze. This modularity is a key advantage of our method: Step~\ref{step1} accommodates different MPD instantiations (including simple parametric ones, should the data warrant it), while Step~\ref{step2} accommodates different DBC procedures depending on the complexity of the problem. We now discuss Steps~\ref{step1} and \ref{step2} in detail.

\subsection{Density uncertainty}

Step~\ref{step1} consists in obtaining MPD samples of the data-generating density. Concretely, assume we observe $X_{1:n}$ taking values in $\mathcal{X}\subseteq\mathbb R^p$ and generated i.i.d.\ from an unknown density $f_*$ with respect to the Lebesgue measure $\lambda$ on $\mathcal{X}$. The MPD framework treats uncertainty as arising from the unobserved tail of the data sequence $X_{n+1:\infty}$ \citep{fong2023martingale}; if the entire sequence were available, any identifiable parameter would be pinned down with certainty. Because the tail is missing, uncertainty is accounted for by imputing it recursively: one first draws $Y_{1}\sim\pi_{n,1}(\cdot\mid X_{1:n})$ to impute $X_{n+1}$, then $Y_{2}\sim\pi_{n,2}(\cdot\mid X_{1:n},Y_1)$ to impute $X_{n+2}$, and so on. Therefore, for any $N\in\mathbb N$, we replace the missing $X_{n+1:N}$ with $(Y_1,\ldots,Y_N)$. The key ingredient of this procedure, referred to as predictive resampling, is a sequence $(\pi_{n,k})_{k\in\mathbb N}$ of predictive distributions driving the recursive imputation. Once the entire sequence has been resampled, any identifiable parameter is a deterministic function of it. The randomness induced by resampling translates into randomness over the parameter, which is said to be drawn from the MPD, denoted as $\Pi_n$.

The conceptual appeal of this framework is that it generalizes classical Bayesian learning. As \citet{fong2023martingale} show, if $(\pi_{n,k})_{k\in\mathbb N}$ coincides with the sequence of predictive distributions arising from a standard prior-likelihood-posterior pipeline, the MPD reduces to the traditional Bayesian posterior. Bayesian procedures thus treat uncertainty as fundamentally arising from prediction of unseen data \citep{definetti1937prevision,fong2023martingale,fortini2023prediction,fortini2025exchangeability}, and tackling prediction directly yields a philosophically coherent notion of uncertainty that, as a welcome byproduct, replaces costly sequential MCMC with resampling from the predictive rule, which can be fully parallelized across independent samples.

The predictive resampling strategy we focus on is described as follows. Suppose that, based on the observed data $X_{1:n}$, we have trained a differentiable density estimator $\{f_\theta(x):\theta\in\Theta\}$ with parameter space $\Theta\subseteq\mathbb R^d$, where each $f_\theta(\cdot)$ is a probability density on $\mathcal X$ with respect to $\lambda$. Let $s(x;\theta):=\nabla_\theta \log f_\theta(x)$ denote the associated score function, and let $\theta_{n,0}$ be the trained parameter (e.g., obtained via gradient-based  log-likelihood maximization). Fix the learning rate schedule $\eta_{n,k}:=\eta_0/(n+k-1)$ for some $\eta_0>0$, and consider the scheme
%
 $$   Y_k \mid Y_{1:k-1} \sim f_{\theta_{n,k-1}}, \quad \quad
    \theta_{n,k} = \theta_{n,k-1} + \eta_{n,k}s(Y_k;\theta_{n,k-1}),$$
%
for all $k\geq 1$, with $Y_1\sim f_{\theta_{n,0}}$ when $k=1$ \citep{holmes2023statistical,cui2025martingale, fortini2025exchangeability}. Intuitively, $(\theta_{n,k})_{k\in\mathbb N}$ forms a martingale because of the score identity $\mathbb E_{Y\sim f_\theta}[s(Y;\theta)]=0$, which holds under mild regularity assumptions on the density model. The limiting distribution $\Pi_n$ of this martingale is the MPD of interest, which is approximated in practice by stopping the resampling after a large but finite number of steps $N$.

For the MPD to be well-defined, the resampling procedure must converge, i.e., the random limit $\theta_{n,\infty}:=\lim_{N\to\infty}\theta_{n,N}$ must exist. To formalize this, we introduce an algorithmic probability space $(\Omega,\mathcal F,\mathbb P)$ on which the resampling is defined and governed by $\mathbb P$. Moreover, we let $(\mathcal F_k)_{k\in\mathbb N}$ denote the filtration generated by $(Y_k)_{k\in\mathbb N}$, with respect to which $(\theta_{n,k})_{k\in\mathbb N}$ is adapted.\footnote{We refer to \citet{cui2025martingale} for a detailed discussion of the $1/n$-order variance at the parameter level implied by this procedure with our chosen step size schedule. We refrain from discussing this further, as our focus is on density rather than parameter estimation.}

\begin{theorem}\label{thm:mart_post_growth}
Assume that
\begin{enumerate}
    \item[(i)] $\mathbb{E}_{Y \sim f_\theta}[s(Y;\theta)] = 0$ for all $\theta \in \Theta$;
    \item[(ii)] There exist constants $A, B > 0$ such that $\mathbb{E}_{Y \sim f_\theta}[\|s(Y;\theta)\|^2] \le A + B\|\theta\|^2$ for all $\theta \in \Theta$.
\end{enumerate}
Then, for any initialization $\theta_{n,0}\in\Theta$, the sequence $(\theta_{n,k})_{k \in\mathbb N}$ is an $L^2$-bounded martingale under $\mathbb{P}$, and $\theta_{n,\infty} := \lim_{k\to\infty}\theta_{n,k}$ exists $\mathbb{P}$-almost surely and in $L^2$.
\end{theorem}

Theorem~\ref{thm:mart_post_growth} provides conditions under which predictive resampling converges, enabling the formal definition of the MPD $\Pi_n$ as the push-forward measure of $\theta_{n,\infty}$ through the resampling distribution $\mathbb P$; that is, $\Pi_n:=\mathbb P\circ\theta_{n,\infty}^{-1}.$ This definition immediately induces a corresponding MPD at the level of the density via the mapping $\theta\mapsto f_\theta$, allowing to capture the density uncertainty of interest.

Score-based MPDs are appealing because they depend solely on the ability to sample from $f_\theta$ and to evaluate $\nabla_\theta\log f_\theta(x)$, which aligns well with modern deep learning solutions based on automatic differentiation. A particularly relevant class of neural density estimators consists of \emph{normalizing flows} \citep{rezende2015variational,papamakarios2017masked,papamakarios2019normalizing}, which we exploit in our applications. These models construct flexible densities through the change-of-variables formula applied to a simple baseline distribution, typically a standard normal. Formally, let $Z$ denote an $\mathbb{R}^p$-valued random variable with standard multivariate normal density $p_Z(z) = (2\pi)^{-p/2}\exp(-\|z\|^2/2)$. Normalizing flows define a family of diffeomorphisms $D_\theta : \mathbb{R}^p \to \mathbb{R}^p$ parameterized by $\theta \in \Theta \subseteq \mathbb{R}^d$ (and typically neural network-based), with inverse mappings $G_\theta(y) := D_\theta^{-1}(y)$. The associated density estimator follows from the change-of-variables formula:
\begin{equation}\label{eq:nf_density}
    f_\theta(y) = p_Z(G_\theta(y)) \left| \det \partial_y G_\theta(y) \right|,
\end{equation}
where $\partial$ denotes the Jacobian operator. Beyond the ease of computing scores and their efficient training via gradient-based maximum likelihood, normalizing flows offer the additional advantage that Equation~\eqref{eq:nf_density} enables density evaluation at any parameter value $\theta\in\Theta$ and data point $x\in\mathcal X$, which is essential given our downstream interest in DBC.

\begin{theorem}\label{thm:nf_compliance}
Assume that $f_\theta$ defined in \eqref{eq:nf_density} satisfies the following conditions:
\begin{enumerate}
    \item[(C1)] For all $y \in \mathbb{R}^p$, the maps $\theta \mapsto G_\theta(y)$ and $\theta \mapsto \det \partial_y G_\theta(y)$ are continuously differentiable on $\Theta$.
    \item[(C2)] There exists an integrable function $g: \mathbb{R}^p \to \mathbb{R}$ such that $\|\nabla_\theta f_\theta(y)\| \le g(y)$ for all $\theta \in \Theta$ and $y \in \mathbb{R}^p$.
    \item[(C3)] There exist functions $h_1, h_2: \mathbb{R}^p \to \mathbb{R}_+$ satisfying $\mathbb{E}_{Z \sim p_Z}[\|Z\|^2 h_1(Z)^2 + h_2(Z)^2] < \infty$ and such that, for all $z \in \mathbb{R}^p$ and $\theta \in \Theta$, the following hold:
    \begin{align*}
        \Big\| \partial_\theta G_\theta(y) \big|_{y = D_\theta(z)} \Big\|_F &\le h_1(z)(1 + \|\theta\|), \\
        \Big\| \nabla_\theta \log \left| \det \partial_y G_\theta(y) \right| \big|_{y = D_\theta(z)} \Big\| &\le h_2(z)(1 + \|\theta\|),
    \end{align*}
    where $\|A\|_F := \big(\sum_{i,j} A_{ij}^2\big)^{1/2}$ denotes the Frobenius norm of a matrix $A$.
\end{enumerate}
Then $\{f_\theta:\theta\in\Theta\}$ satisfies assumptions (i) and (ii) of Theorem~\ref{thm:mart_post_growth}, and the resulting score-based MPD is well-defined.
\end{theorem}

Theorem~\ref{thm:nf_compliance} provides general conditions under which normalizing flow architectures generate a valid score-based MPD. As a final remark on the question of resampling convergence, we note that while such analytical guarantees are appealing, the generality of the score-based MPD framework for density uncertainty extends beyond models where conditions like the above can be verified analytically. In the presence of complex models (e.g., for specific neural network architectures parametrizing $D_\theta$) one may, for practical purposes, simply verify that resampling stabilizes for large enough $N$, as we have done in the applications presented below.

\begin{figure}
    \centering
    \includegraphics[width=0.99\textwidth]{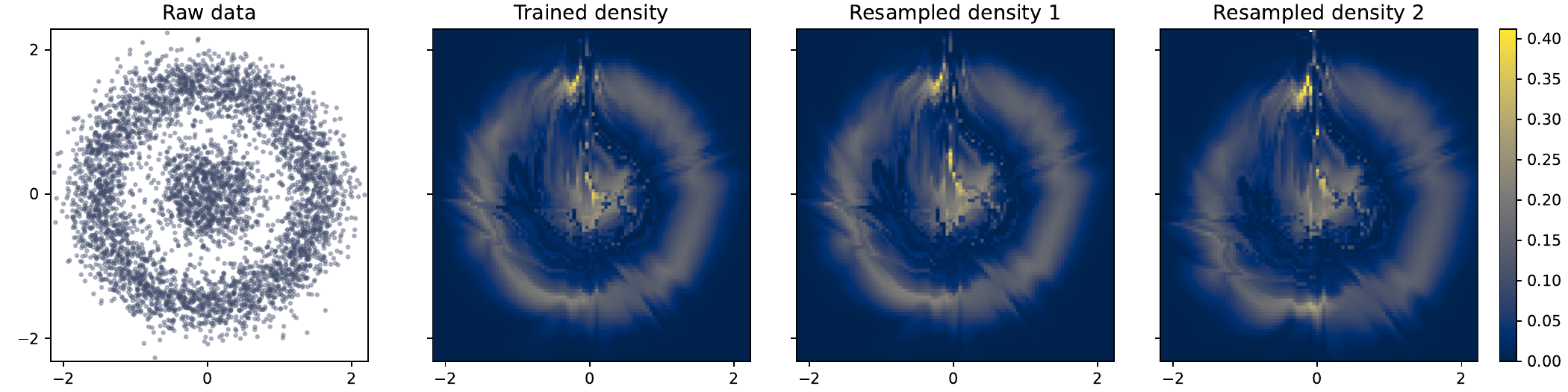}
    \caption{Illustration of the proposed methodology on a noisy concentric circles dataset. From left to right: observed data; trained density estimate; two MPD samples of the density obtained via score-based predictive resampling.}
    \label{fig:circles_combined_density}
\end{figure}

\begin{example}\label{ex:circle_density}
    Figure~\ref{fig:circles_combined_density} illustrates the density estimation and resampling steps on a two-dimensional dataset generated from two noisy concentric circles. This is a setting where DBC will be useful because of the irregular cluster shapes, and where the ring-like, effectively one-dimensional structure near which the data concentrates provides a simple illustration of the manifold hypothesis. For the density uncertainty step, we train a Masked Autoregressive Flow \citep[MAF;][]{papamakarios2017masked} density estimator on $n=5{,}000$ observations and generated $T=1{,}000$ independent score-based MPD resamples of the density (each via $N=3{,}000$ updates). The second panel shows the baseline density estimate, which correctly identifies the low-density region separating the two dense rings, while the last two panels show independent score-based MPD draws, each a plausible perturbation of the baseline estimate. 
\end{example}

\subsection{Clustering uncertainty}
Once MPD uncertainty is obtained in the form of density samples $f^{(1)},\ldots,f^{(T)}$, it can be directly propagated to the density-based clustering structure by applying the DBC procedure of interest to each $f^{(j)}$ in parallel. As mentioned, a large number of DBC methods exists with varying degrees of complexity and flexibility. For the moment, we focus on perhaps the most fundamental one, which roughly corresponds to the seminal DBSCAN algorithm \citep{ester1996density} and forms the basis for our theoretical guarantees in the next section. Given a density $f: \mathcal{X} \to [0,\infty)$ and a level $c \in (0,\infty)$, we let $\mathcal{L}_c(f) := \{x \in \mathcal{X} : f(x) \geq c\}$ denote the upper-level set of $f$ at level $c$. We define $\mathcal{C}_c(f)$ as the set of path-connected components (maximal path-connected subsets) of $\mathcal{L}_c(f)$.\footnote{A subset $A$ of $\mathcal{X}$ is path-connected if for any two points $x,y\in A$ there exists a continuous path $\gamma:[0,1]\to\mathcal{X}$ connecting them (i.e., $\gamma(0)=x$ and $\gamma(1)=y$) that lies entirely in $A$ (i.e., $\gamma(w)\in A$ for all $w\in[0,1]$).} The elements of $\mathcal{C}_c(f)$ are called the clusters of $f$ at level $c$, and $k_f(c) := |\mathcal{C}_c(f)|$ denotes the number of such clusters. We refer to this procedure as \emph{level set clustering}.

Level set clustering acts purely at the level of an input density $f$ and outputs a partition of the sample space according to the path-connected components of $\mathcal{L}_c(f)$. In particular, within-cluster homogeneity is embodied by the existence of a continuous path, running through a high-density region, connecting any two members of the same cluster, while separation by low-density regions operationalizes across-cluster heterogeneity. As we show in the next section, this density-based approach facilitates the establishment of theoretical properties, since the inference target is fully determined by the underlying density and estimation of the latter is supported by a well-developed area of statistical research. However, from a computational perspective, identifying the path-connected components exactly is infeasible in more than one or two dimensions, leading to the need for approximation methods. In particular, path-connectedness is often approximated by (i) choosing representative sample points in the sample space (typically the training data, which is often what one seeks to cluster anyway, perhaps with the addition of new out-of-sample points), and (ii) building a neighborhood graph with vertices at those points and edges between close points, such as a $\delta$-neighborhood or $k$-nearest-neighborhood graph \citep{wang2019dbscan}. Additionally, small spurious clusters arising from the approximate nature of the algorithm must be addressed.

In Algorithm~1 in Appendix~\ref{app:experiments}, we provide full details on our preferred implementation, which can be summarized as follows: given a level $c>0$ and a neighborhood radius $\delta>0$, for each MPD sample $f^{(j)}$ perform the following steps:
\begin{enumerate}
    \item Compute $f^{(j)}(x)$ at all $x$ in the finite set of points to cluster;
    \item Let $\mathcal{D}^{(j)}$ be the collection of points $x$ such that $f^{(j)}(x)\geq c$;
    \item Build a $\delta$-neighborhood graph $G^{(j)}$ with vertex set $\mathcal{D}^{(j)}$; that is, given an arbitrary ordering $x_1,\ldots,x_{|\mathcal D^{(j)}|}$ of $\mathcal D^{(j)}$, set $G^{(j)}(i,k)=1_{\{\|x_i-x_k\|<\delta\}}$ for any $i,k=1,\ldots,|\mathcal D^{(j)}|$;
    \item Cluster together the data points that are connected by a path in $G^{(j)}$.\footnote{For our upcoming demonstrations with level set clustering, we also apply a final (but optional) step: each point outside $\mathcal D^{(j)}$ is clustered together with its closest neighbor in $\mathcal D^{(j)}$. Based on the task at hand and the intended interpretation, one may also skip this step and label points outside $\mathcal D^{(j)}$ as ``noise.''}
\end{enumerate}
A key feature of this procedure is that it relies only on density evaluations and a graph encoding a discrete version of the target topology on $\mathcal{X}$, which extends applicability beyond the case where $\mathcal{X}\subseteq\mathbb{R}^p$. The dependence on only these two ingredients is shared by many other DBC procedures, including ToMATo \citep{chazal2013persistence,gudhi:urm}, the algorithm we employ in the real-world applications discussed in Section~\ref{sec:applications} below.

\begin{figure}
    \centering
    \includegraphics[width=0.99\textwidth]{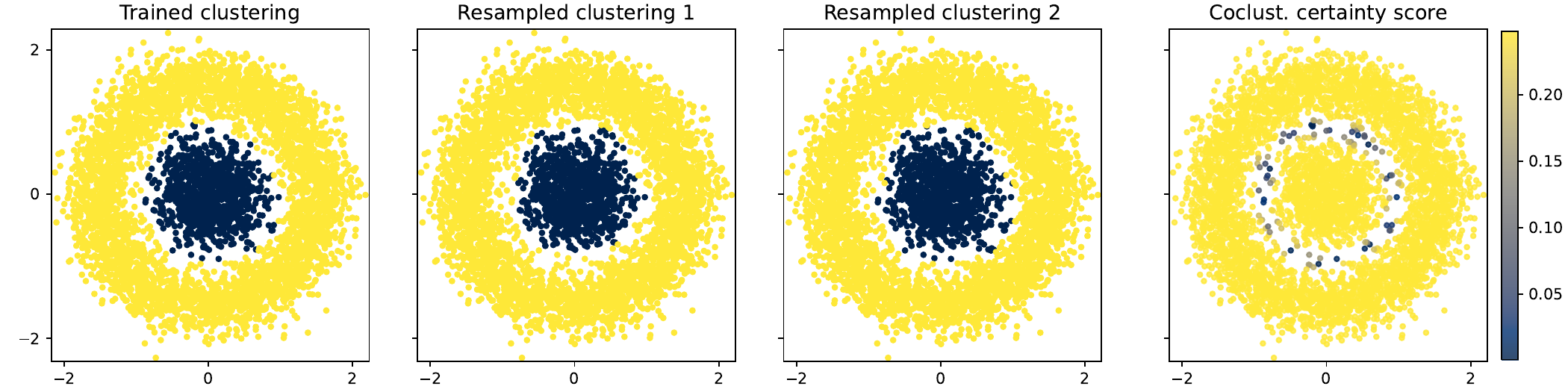}
    \caption{MPD Clustering uncertainty for the noisy concentric circles dataset. Coloring is by cluster membership in the three left-most plots, and by co-clustering certainty score in the right plot.}
    \label{fig:circles_combined_clust}
\end{figure}

\begin{example}\label{ex:circle_clust}
Continuing Example~\ref{ex:circle_density}, we apply level set clustering\footnote{The level $c$, common across all trained and resampled densities, was chosen to have 10\% of the data with trained density value below $c$.} to each of the $1{,}000$ resampled densities, yielding a point estimate from the baseline trained density and MPD samples of the clustering structure. Figure~\ref{fig:circles_combined_clust} shows the resulting clustering uncertainty: the first three panels display cluster membership for the trained and two resampled densities, which all recover the two-ring structure. To visualize MPD clustering uncertainty, we compute the co-clustering matrix $\mathcal{M}$, where $\mathcal{M}(i,j)$ records the proportion of MPD samples in which data points $i$ and $j$ are assigned to the same cluster. The right panel of Figure~\ref{fig:circles_combined_clust} visualizes a simple derived measure of point-wise co-clustering certainty:
\(
\mathcal{S}(i):=n^{-1}\sum_{j=1}^{n} (\mathcal{M}(i,j)-0.5)^2,
\)
where low values (rows $\mathcal{M}(i,\cdot)$ averaging close to 0.5) indicate high uncertainty across MPD samples. Data near the boundary between the two circles exhibit the highest uncertainty, showing that our method effectively captures ambiguity in the clustering structure of the noisy dataset.
\end{example}

\section{Theoretical guarantees}\label{sec:consistency}

Now we study the asymptotic properties of our uncertainty quantification when DBC is performed through level set clustering as described in the previous section. For this purpose, we explicitly adopt the frequentist assumption that the data are i.i.d.\ from an unknown distribution $F_*$ with density $f_*$, and denote by $F_*^n$ the $n$-fold product measure induced by $F_*$. At a high level, our proposed methodology comprises three steps: computing a frequentist density estimator $f_{\theta_{n,0}}$, obtaining a score-based martingale posterior at the level of the density, and translating that into clustering uncertainty via DBC. Our strategy for obtaining theoretical convergence guarantees reflects this process: we start from convergence of the estimator $f_{\theta_{n,0}}$, then establish contraction of the MPD towards $f_*$, and finally ensure consistency of the induced clustering.

To that end, we begin with the following definition, which relies on the choice of a metric $d$ between probability densities (e.g., Hellinger or $\mathcal{L}^p, p\in[1,\infty]$, metrics).

\begin{definition}\label{def:frequentist_consistency}
    For a positive vanishing sequence $(\varepsilon_n)_{n\in\mathbb{N}}$, we say that $f_{\theta_{n,0}}$ is $(d,\varepsilon_n)$-consistent at $f_*$ if
    \[
    \lim_{n\to\infty} F_*^n\!\left(d\big(f_{\theta_{n,0}},f_*\big)>\varepsilon_n\right)=0.
    \]
\end{definition}

Definition~\ref{def:frequentist_consistency} formalizes the classical notion of convergence rate (in probability) of a frequentist density estimator $f_{\theta_{n,0}}$ relative to a metric $d$. An extensive literature exists on the foundational theory of convergence rates for nonparametric estimators and general M-estimators across diverse model families \citep[see, e.g.,][]{vandeGeer2000empirical, tsybakov2009introduction}, which may be borrowed to replace $\varepsilon_n$ with an explicit rate. However, these results have not yet been extended to complex models such as the MAF density estimators we adopt in practice. Because this extension is beyond the scope of the present work, all of the following results are stated implicitly in terms of an abstract rate $\varepsilon_n$. Moreover, a key feature of our framework, as our upcoming results show, is that once a convergence rate for the baseline estimator $f_{\theta_{n,0}}$ is available, contraction of the MPD $\Pi_n$ can be entirely decoupled from the specific model class.

\begin{theorem}\label{thm:density_rates_local}
Let $\mathcal{V}_\delta(f_*) := \{\theta \in \Theta : d(f_\theta, f_*) \le \delta\}$. Assume that $\mathbb{E}_{Y \sim f_\theta}[s(Y;\theta)] = 0$ for all $\theta \in \Theta$ and that there exist constants $L>0$, $\gamma\in(0,1]$, $\delta>0$, and $M_\delta<\infty$ such that:
\begin{enumerate}
    \item[(I)] $d(f_\theta,f_{\theta'}) \le L\|\theta-\theta'\|^\gamma$ for all $\theta,\theta'\in\Theta$,
    \item[(II)] $\sup_{\theta \in \mathcal{V}_\delta(f_*)} \mathbb{E}_{Y \sim f_\theta}[\|s(Y;\theta)\|^2] \le M_\delta$.
\end{enumerate}
If $f_{\theta_{n,0}}$ is $(d,\varepsilon_n)$-consistent at $f_*$, then
\begin{equation}\label{eq:posterior_rate_local}
\Pi_n\big(\theta: d(f_\theta,f_*)>\tilde\varepsilon_n\big)\to0 \quad \text{in } F_*^n\text{-probability}
\end{equation}
for any decreasing sequence $\tilde\varepsilon_n\gtrsim\varepsilon_n$ satisfying $\lim_{n\to\infty} n\tilde\varepsilon_n^{2/\gamma}=\infty$.\footnote{The symbol $\gtrsim$ denotes inequality up to a multiplicative constant. Also, see \citet{fong2024asymptotics,fong2025bayesian} for related results.}
\end{theorem}

Theorem~\ref{thm:density_rates_local} links the rate $\varepsilon_n$ to contraction of the MPD towards $f_*$. Specifically, the contraction rate of $\Pi_n$ in the metric $d$ is mediated by the smoothness of the parametrization $\theta\mapsto f_\theta$, with greater smoothness yielding a faster rate. We also note that, while we do not report these extensions, inspection of our proof in Appendix~\ref{app:proofs} shows that strengthening the result to almost sure contraction (given almost sure convergence of the baseline estimator) and weakening it to simple consistency (replacing $\tilde\varepsilon_n$ with any constant $\varepsilon>0$) are both straightforward extensions that we do not pursue explicitly.

We now turn to consistency of the MPD applied to level set clustering. Preliminarily, we fix the density level $c>0$ and omit it from the notation. Without loss of generality, we also treat $\Pi_n$ as a distribution over densities rather than parameters, making it possible to apply our result to any MPD (not just our score-based instantiation) for which density contraction has been established. From now on, we also implicitly assume that the true density $f_*$ and all densities in the model class belong to a regular class of functions whose level sets have a finite number of path-connected components, and denote by
\(
\mathcal{C}(f_*)=\{C_1(f_*),\ldots,C_{k_{f_*}}(f_*)\}
\)
the true clusters. Throughout, let $A\triangle B := (A\setminus B)\cup(B\setminus A)$ denote the symmetric difference between any two sets $A$ and $B$, $\|f-f'\|_{\mathcal{L}^\infty}:=\sup_{x\in\mathcal{X}}|f(x)-f'(x)|$ the sup-norm distance between any two densities $f$ and $f'$, and $S_k$ (where $k\in\mathbb N$) the set of permutations of $\{1,\ldots,k\}$. Moreover, we state the following conditions on the true density for future reference. 

\begin{assumption}\label{ass1}
For any two disjoint sets $A, B \subset \mathcal{X}$, define
\begin{equation*}
    S(A, B) := \sup_{\gamma \in \Gamma(A, B)} \inf_{t \in [0, 1]} f_*(\gamma(t)),
\end{equation*}
where $\Gamma(A, B)$ is the set of all continuous paths $\gamma: [0, 1] \to \mathcal{X}$ with $\gamma(0) \in A$ and $\gamma(1) \in B$. We then consider the following four conditions, where $\operatorname{dist}(x,A):=\inf_{a\in A}\|x-a\|$ denotes the distance from a point $x\in\mathcal X$ to a set $A\subset\mathcal X$:
\begin{enumerate}
        \item[(a)] There exist $C_0,\varepsilon_0>0$ such that, for all $\varepsilon\in(0,\varepsilon_0)$, $\lambda(\{x\in\mathcal{X} : |f_*(x)-c|\leq \varepsilon\})\leq C_0\varepsilon$.
        \item[(b)] There exists $\delta_0\in(0,c)$ such that $\max_{i \neq j} S(C_i(f_*), C_j(f_*))\le c - \delta_0$.
        \item[(c)] There exists $\eta_0 > 0$ such that for all $\eta\in(0, \eta_0)$ and each $j \in \{1, \ldots, k_{f_*}\}$, the set $K_j^\eta := \{x \in C_j(f_*) : f_*(x) \geq c + \eta\}$ is non-empty and path-connected.
        \item[(d)] There exists $\rho_0>0$ such that, letting $N_j:=\{x\in\mathcal X:\operatorname{dist}(x,C_j(f_*))\le\rho_0\}$ for $j=1,\ldots,k_{f_*}$, the following hold for the same choice of $\eta_0$ as in point (c):
        \begin{itemize}
            \item The sets $N_1,\ldots,N_{k_{f_*}}$ are pairwise disjoint;
            \item $\sup\Big\{f_*(x) : x\in\mathcal X\setminus\bigcup_{j=1}^{k_{f_*}}N_j\Big\}\le c-\eta_0$;
            \item For every $j\in\{1,\ldots,k_{f_*}\}$ and every $x\in N_j$ with $f_*(x)>c-\eta_0$, there exists a continuous path $\gamma : [0,1]\to N_j$ with $\gamma(0)=x$, $\gamma(1)\in K_j^{\eta_0/2}$, and $t\mapsto f_*(\gamma(t))$ non-decreasing on $[0,1]$.
        \end{itemize}
\end{enumerate}
\end{assumption}

In words, condition (a) requires that $f_*$ does not flatten at level $c$, condition (b) requires the true clusters to be separated by sufficiently deep density ``valleys,'' condition (c) requires a minimal degree of stability of the true clusters under perturbations of the level $c$, and condition (d) requires that every point within a fixed distance $\rho_0$ of a true cluster, and not too far below level $c$, can be connected to that cluster's core by a path along which $f_*$ never decreases, and that no other region of the sample space comes close to level $c$. With these conditions established, a first estimable object of interest is the level set $\mathcal{L}(f)$, which we treat by following \cite{cuevas2006plug}.

\begin{theorem}\label{thm:level_set_consistency}
    Assume $f_*$ satisfies condition (a) in Assumption~\ref{ass1}. If $\tilde{\varepsilon}_n$ is a MPD contraction rate as in \eqref{eq:posterior_rate_local} for $d(f,f')=\|f-f'\|_{\mathcal{L}^\infty}$, then
    \begin{equation*}
        \Pi_n\big(f : \lambda\big(\mathcal{L}(f)\triangle \mathcal{L}(f_*)\big)>2C_0\tilde{\varepsilon}_n\big) \to 0\quad \text{in } F^n_*\text{-probability}.
    \end{equation*}
\end{theorem}

Theorem~\ref{thm:level_set_consistency} reveals that MPD contraction in the sup-norm metric implies the same rate for estimating the level set, using the Lebesgue measure of the symmetric difference as a discrepancy. As this result and the following show, controlling density-based clustering requires strong control over the underlying density; hence the requirement for sup-norm contraction.

Convergence of the upper-level set alone is generally not sufficient to ensure consistency at the level of the whole clustering structure, as discrepancies among components may persist. For instance, sup-norm closeness of $f$ to $f_*$ alone does not prevent $f$ from exhibiting small, spurious path-connected components of $\mathcal{L}(f)$ in regions where $f_*$ lies just below $c$, so that $k_f$ need not equal $k_{f_*}$. To handle this kind of behavior, we work with a persistence-based refinement of the estimated clusters, which is similar in spirit to the ToMATo algorithm used in the applications of Section~\ref{sec:applications}. 

\begin{definition}
    For all $C\in\mathcal{C}(f)$, let $M(C,f):=\sup_{x\in C}f(x)$ and fix $\eta>0$. Then $C$ is $\eta$-persistent if the path-connected component of $\mathcal{L}_{c-\eta}(f)$ containing $C$ contains no other $C'\in\mathcal{C}(f)$ with $M(C',f)>M(C,f)$.
\end{definition}

In light of the above definition, we write $\mathcal{C}^\eta(f)\subseteq\mathcal{C}(f)$ for the collection of $\eta$-persistent clusters of $f$ at level $c$, and $k_f^\eta:=|\mathcal{C}^\eta(f)|$ for their number. Moreover, we measure the discrepancy between two collections of subsets of $\mathcal{X}$ of equal cardinality $k$, say $\mathcal{C}=\{C_1,\ldots,C_k\}$ and $\mathcal{C}'=\{C_1',\ldots,C_k'\}$, by
\[
\mathcal{D}(\mathcal{C},\mathcal{C}')
:=
\inf_{\sigma\in S_{k}}
\sum_{j=1}^{k}
\lambda\big(C_j\triangle C_{\sigma(j)}'\big).
\]
We can now turn to the consistency of $\Pi_n$ for the persistent clusters just introduced.

\begin{theorem}\label{thm:clustering_consistency}
Assume $f_*$ satisfies conditions (a)--(d) in Assumption~\ref{ass1}. If $\tilde{\varepsilon}_n$ is a MPD contraction rate as in \eqref{eq:posterior_rate_local} for $d(f,f')=\|f-f'\|_{\mathcal{L}^\infty}$, then
\begin{equation}\label{eq:consistency_number_clusters}
\Pi_n\big(f:k_f^{4\tilde\varepsilon_n}\neq k_{f_*}\big)\to0
\quad \text{in } F_*^n\text{-probability}
\end{equation}
and
\begin{align}\label{eq:consistency_clusters}
\Pi_n\bigg(
f: k_f^{4\tilde\varepsilon_n}=k_{f_*},\,
\mathcal{D}(\mathcal{C}^{4\tilde\varepsilon_n}(f),\mathcal{C}(f_*))
> k_{f_*} C_0 \tilde{\varepsilon}_n
\bigg)\to0
\quad \text{in } F_*^n\text{-probability}.
\end{align}
\end{theorem}

Theorem~\ref{thm:clustering_consistency} is the main result of our theoretical analysis, as it establishes MPD contraction rates for DBC instantiated as (persistence-filtered) level set clustering. In particular, the number of clusters $k_{f_*}$ is consistently estimated by the number of $4\tilde\varepsilon_n$-persistent clusters, and the same contraction rate as in sup-norm density estimation is achieved for the entire clustering structure, under appropriate regularity conditions on the true density and its clustering structure. Moreover, note that, while the modification of level set clustering via $\eta$-persistence is key for our proofs to work in excluding spuriously estimated clusters, the degree of such a modification decreases as $n$ grows (as $\lim_{n\to\infty}4\tilde \varepsilon_n= 0$), effectively recovering the original procedure in the limit.

We conclude this section by noting that, while we focused on level set clustering for technical convenience, inspecting our proofs in Appendix~\ref{app:proofs} clarifies that the key to clustering consistency is sup-norm control of the estimated density. The latter is a central topic in statistical estimation theory; see for instance \citet{silverman1978weak,ibragimov1980estimate,ibragimov2013statistical,gine2009uniform,gine2016mathematical}. Once sup-norm convergence is verified, level set clustering consistency follows under mild additional conditions on $f_*$. We suspect that for other DBC algorithms that take the estimated density as an input, sup-norm control will also suffice for consistency of the induced clustering structure. For instance, our proof readily extends to consistency at multiple, but finitely many level set clusterings (for different values of $c$), which is at the basis of popular hierarchical level set clustering algorithms \citep{campello2013density}. We refer to \citet{rinaldo2010generalized,rinaldo2012stability,wang2019dbscan} for further discussion and additional useful proof techniques.

\section{Applications}\label{sec:applications}

In this section, we illustrate the methodology on two real-world datasets consisting of handwritten digit images and single-cell RNA sequencing records. Before doing so, we briefly discuss the choice of an important hyperparameter for the score-based MPDs, namely the base learning rate $\eta_0>0$. Intuitively, this parameter controls the spread of the MPD around the baseline frequentist estimator, although this interpretation is only heuristic for the complex models we consider. Crucially, since $\eta_0$ determines the step size of the gradient updates in the predictive resampling procedure, it must be chosen to ensure numerical stability, avoiding gradient explosion, while still allowing sufficient exploration of the space of densities to capture meaningful uncertainty. Step-size selection is itself a classical topic in the literature on stochastic approximation and gradient-based optimization \citep{robbins1951stochastic,polyak1992acceleration,bottou2018optimization}, and our practical recommendations below are informed by standard heuristics from that area.

In practice, we adopt the following strategy. First, we initialize the resampling procedure with a step size $\eta_0/n$ of the same order as that used by the optimizer employed to train the baseline estimator. Second, we adjust $\eta_0$ to ensure numerical stability. For simpler models, it would also be possible to incorporate second-order information, for instance by choosing a step size aligned with the inverse Fisher information matrix, which can lead to faster convergence \citep{fong2024asymptotics}. However, this results in a modified algorithm that is not feasible for our large-scale, neural-network-powered applications.

\subsection{Handwritten digits}

Our first application demonstrates a key strength of DBC: the ability to assign previously unseen points to clusters without retraining. We use the MNIST dataset \citep{lecun1998gradient}, focusing on $n=5{,}000$ training images and a held-out set of $1{,}000$ out-of-sample (OOS) images, restricted to digits 3 and 8. These two classes are deliberately chosen because they are topologically similar (both consisting of two vertically-aligned loops), so that genuine ambiguity is expected and to be captured by the MPD-based uncertainty quantification.

As a preprocessing step, we embed the images into a latent space of dimension $p=24$ using a convolutional autoencoder fitted to the training data. The same type of MAF density estimator used in Example~\ref{ex:circle_density} is then fitted to the encoded training data, and clustering is performed on the combined set of training and OOS encoded images using the ToMATo algorithm \citep{chazal2013persistence,gudhi:urm}, a DBC method that identifies clusters as persistent connected components across level sets.\footnote{More specifically, ToMATo can be viewed as a refinement of level set clustering: it slices the target density over a grid of levels and tracks the persistence of connected components across levels. This provides additional robustness to clusters of different prominence and to spurious density peaks due to estimation error. As in standard level set clustering, it requires only density evaluations and a neighborhood graph on the set of points to be clustered. Further implementation details are provided in Appendix~\ref{app:experiments}.} The full pipeline, consisting of baseline training, density resampling (500 independent times), and clustering, took 4 minutes and 24 seconds on a single NVIDIA RTX A4000 GPU. Training and resampling were carried out using the \texttt{surjectors} library \citep{dirmeier2024surjectors} with JAX parallelization \citep{jax2018github}. We emphasize that the class labels are not used during training and serve only for evaluation, making the task substantially more challenging than standard supervised classification on MNIST.

Three findings stand out in Figure~\ref{fig:mnist_combined}. First, the MPD over the number of clusters is sharply concentrated on two-cluster configurations (left panel), with a non-negligible secondary mode at one cluster reflecting the similarity between digits 3 and 8. Second, the inferred clustering largely agrees with the true labels, as visualized in the OOS co-clustering matrix (center panel); moreover, applying the methodology of \citet{bariletto2025conformalized}, the true labeling is contained in a 90\% credible set of partitions under the MPD. Third, the digits with lowest pairwise co-clustering certainty (right panel) correspond to plausibly ambiguous cases, such as 3's with nearly-closed loops or 8's with open loops.

\begin{figure}
    \centering
    \includegraphics[width=0.99\textwidth]{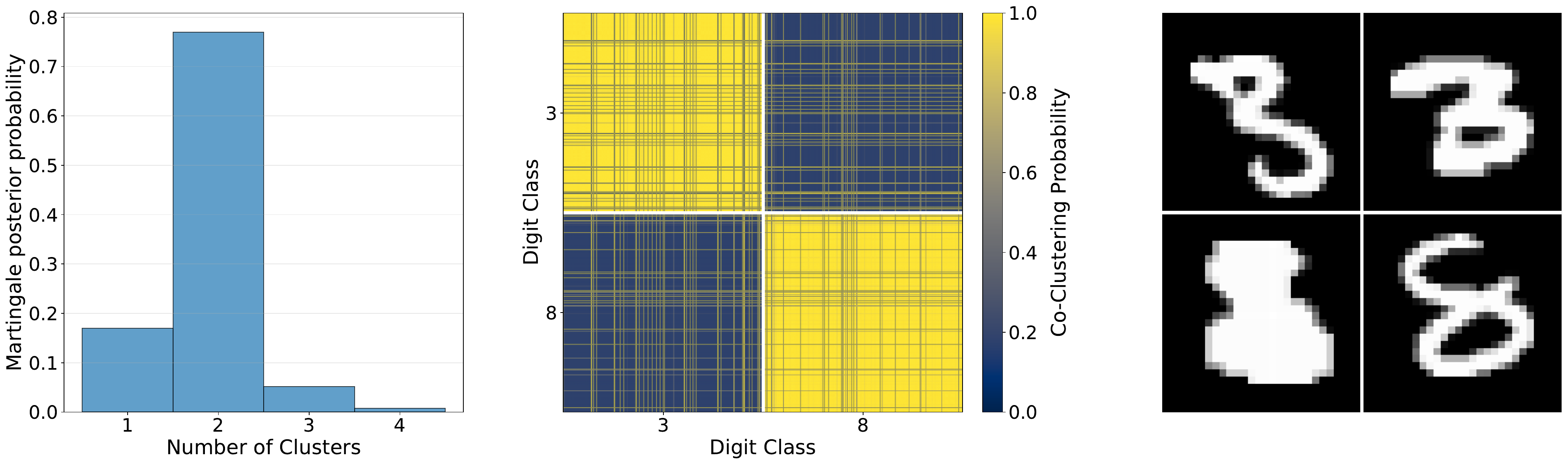}
    \caption{MNIST digits experiment. Left: Martingale posterior distribution of the number of clusters. Center: co-clustering matrix of out-of-sample digits. Right: Selection of digits with low pairwise co-clustering certainty.}
    \label{fig:mnist_combined}
\end{figure}

\subsection{Bone marrow single-cell RNA sequencing data}

The second application illustrates the methodology on data with rich biological structure. We use the Tabula Sapiens Bone Marrow dataset \citep{tabulasapiens2022,quake2025tabula}, a multi-donor single-cell RNA sequencing (scRNA-seq) atlas of human bone marrow comprising $n=27{,}112$ cells annotated with 26 cell types, accessed via the CZ CELLxGENE portal. As is standard for scRNA-seq analyses, we used PCA to reduce dimensionality and retained the first $p=10$ principal components as features. The full pipeline, consisting of baseline training, density resampling (500 independent draws), and clustering, took 4 minutes and 1 second on the same hardware as the MNIST application, illustrating that the method scales to tens of thousands of points with minimal computational overhead. The biological coherence of the results below was confirmed in discussion with domain experts.

Figure~\ref{fig:scrna_pretrained_clust} compares the annotated cell types with the baseline ToMATo clustering derived from our density estimate \citep[using UMAP scores for visualization,][]{becht2019dimensionality}. The density-based approach recovers the major macro-structures of the bone marrow landscape, grouping developmentally related lineages into unified clusters, and identifying fewer clusters than the 26 annotated cell types. This is biologically meaningful, since most cells in the bone marrow are immature precursors existing along a developmental continuum \citep{pellin2019comprehensive}, so the annotated cell types do not correspond to fully separated transcriptional populations.

\begin{figure}
    \centering
    \includegraphics[width=0.99\textwidth]{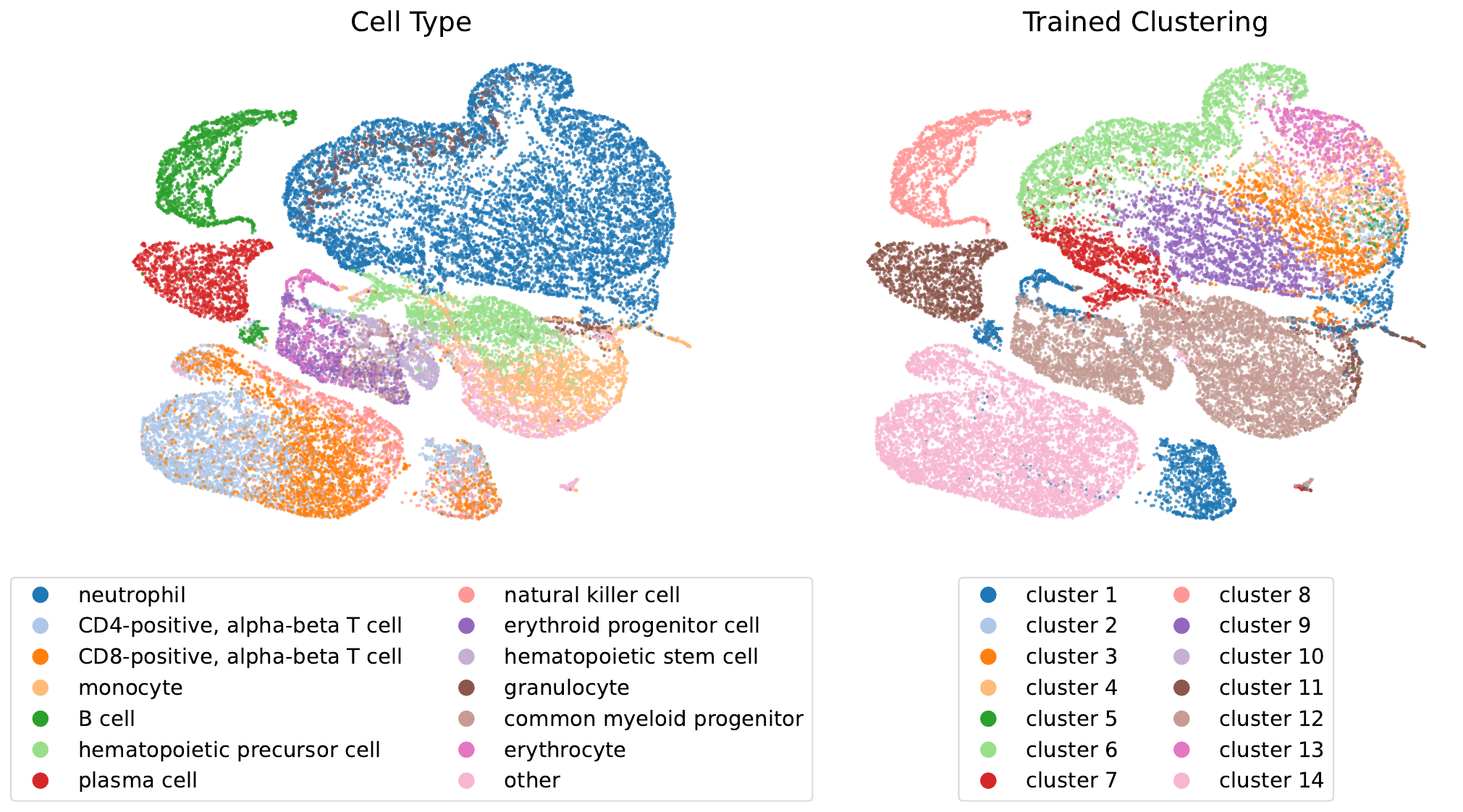}
    \caption{Cell types and trained clustering in UMAP coordinates for bone marrow scRNA-seq data. The 14 trained clusters differ from the 26 annotated cell types in ways that track known biology. Cells from related lymphoid lineages---CD4-positive T cells (light blue), CD8-positive T cells (orange), and natural killer cells (salmon)---are mostly absorbed into a single  (pink) cluster in the lower-left of the embedding. The myeloid and erythroid populations (monocytes, granulocytes, erythroid progenitors, erythrocytes, and common myeloid progenitors, in shades of orange, brown, violet and purple) are similarly grouped together, in line with the developmental continuum of immature precursors that dominate the bone marrow. B cells (dark green) and plasma cells (red), which have distinct transcriptional profiles, are kept as separate clusters. The neutrophil region (dark blue), in contrast, is split across several clusters, reflecting heterogeneity not captured by the annotation and consistent with the different maturation stages neutrophils undergo in the bone marrow. Note that different colors denote different groups within each subplot, but no color correspondence should be inferred across subplots.}
    \label{fig:scrna_pretrained_clust}
\end{figure}

The MPD-based uncertainty quantification results, summarized in Figure~\ref{fig:scrna_combined}, are biologically interpretable in two further ways. First, the posterior co-clustering probabilities aggregated at the cell-type level (left panel) reveal block structures that align with established developmental hierarchies. A stable co-clustering block emerges among T cells (CD4-positive, CD8-positive, regulatory) and Natural Killer cells, consistent with their shared lymphoid lineage. Conversely, B cells and plasma cells form distinct, isolated blocks with minimal cross-type co-clustering, reflecting their well-defined transcriptional identities. The matrix also exhibits a larger, more variable block encompassing various monocytes, macrophages, and progenitor cells (myeloid and erythroid lineages). The elevated uncertainty and cross-type co-clustering in this region is consistent with the fact that these cell types complete their maturation in the periphery, so what is captured in the bone marrow are mainly immature cells whose transcriptional boundaries are diffuse \citep{velten2017human}.

Second, the framework recovers within-type substructure that is not visible in the annotation. In particular, the neutrophil block in the left panel of Figure~\ref{fig:scrna_combined} exhibits a relatively low within-type co-clustering probability compared to other populations, already hinting at internal heterogeneity. The right panel makes this explicit by visualizing the cell-level MPD co-clustering matrix restricted to the 8,677 cells annotated as ``neutrophil,'' revealing approximately seven distinct sub-populations within this single cell type. Unlike T cells or myeloid lineages, neutrophils complete their maturation within the bone marrow, making this population inherently more heterogeneous in this tissue. Specifically, neutrophil maturation proceeds through six well-established morphological stages: Myeloblast $\to$ Promyelocyte $\to$ Myelocyte $\to$ Metamyelocyte $\to$ Band Cell $\to$ Mature Neutrophil \citep{xie2020single}, which is roughly consistent with the number of sub-groups identified by our method. In other words, MPD-based uncertainty quantification surfaces biologically meaningful substructure while maintaining a competitive computational cost on standard GPU hardware.

\begin{figure}
    \centering
    \includegraphics[width=0.99\textwidth]{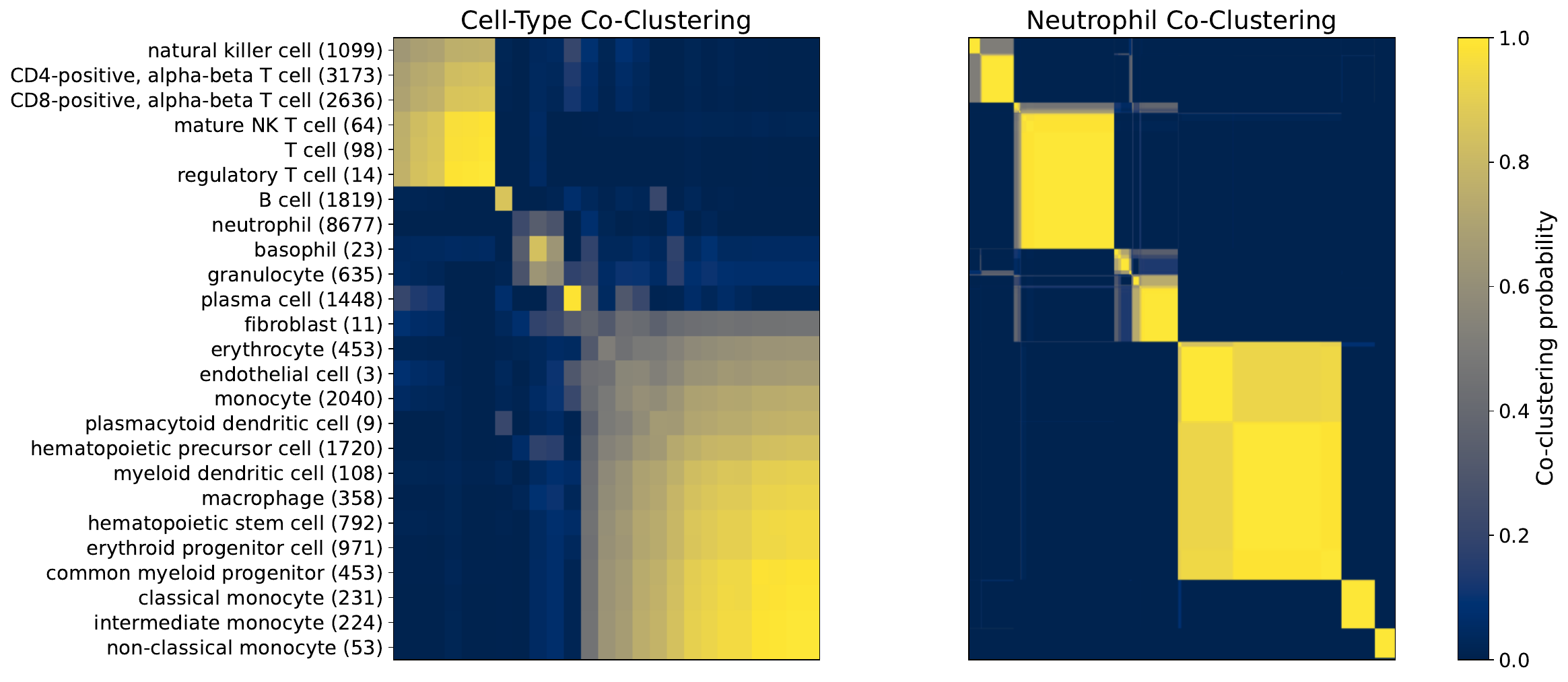}
    
    \caption{MPD co-clustering for bone marrow scRNA-seq data. Left: co-clustering matrix aggregated at the cell-type level (with cell-type count in parentheses). Right: individual-cell level co-clustering matrix restricted to neutrophil cells.}
    \label{fig:scrna_combined}
\end{figure}

\section{Discussion}\label{sec:discussion}

This article develops a framework for uncertainty quantification in density-based clustering. Because DBC treats clusters as functionals of the data-generating density, we propose to propagate uncertainty by applying the clustering procedure to martingale posterior samples of the density obtained through score-based predictive resampling. The result is a posterior distribution over clustering structures that supports principled Bayesian-style inference. In particular, we have demonstrated several practical features of the proposed framework. By relying on density-based clustering, the method accommodates uncertainty for clusters with irregular shapes. The score-based predictive resampling we use for density uncertainty only requires the density estimator to be differentiable, making the framework naturally compatible with modern neural network-based estimators that adapt well to complex and large datasets. The use of martingale posterior updates leads to a fully gradient-based procedure that can be parallelized on modern GPU hardware, as reflected in the competitive computational performance observed in our experiments. The approach can also be applied on top of pre-trained density estimators, requiring only a resampling step to quantify uncertainty. In this sense, our method can be viewed as a scalable procedure for uncertainty quantification wrapped around black-box density estimators that provide point estimates, and density-based applications beyond clustering are fully compatible with our pipeline. From a theoretical perspective, defining clustering as a function of the density yields an identifiable inference target and supports a clean asymptotic analysis. Finally, since clustering is defined on the entire sample space, the method can be used to assign new observations to clusters without repeating the whole training and resampling pipeline, which further enhances the computational efficiency of our method.

Some limitations of the proposed methodology are worth noting. Most importantly, the MPD characterizes uncertainty around a baseline density estimator and therefore inherits its qualitative properties. In particular, a systematically biased baseline will produce a MPD concentrated around the wrong density, regardless of how many resamples are drawn. Moreover, the score-based predictive update we focus on requires the density estimator to admit both efficient sampling and point-wise density evaluation, which excludes some otherwise attractive generative models such as diffusion-based ones \citep{ho2020denoising}.

Several directions for future work remain open. On the methodological side, it would be of interest to extend the framework to density estimators that do not allow for straightforward density evaluation but achieve strong generative performance, such as diffusion models. On the theoretical side, our analysis falls short of providing explicit contraction rates for the deep architectures used in practice, since the underlying frequentist rates for such estimators are themselves an open problem; obtaining those rates and using them as inputs to our framework is a natural next step. Finally, two further useful theoretical extensions would be to rigorously apply our asymptotic analysis beyond level set clustering, and to characterize more refined frequentist (e.g., coverage) properties of martingale posterior distributions for both density estimation and the induced clustering.

\clearpage
\appendix

{\begin{center}
    \LARGE\textbf{Supplementary Material}
\end{center}}

In this Supplementary Material, we collect all the proofs of the theoretical results presented in the main text of the article (Appendix~\ref{app:proofs}) as well as further details on the numerical experiments (Appendix~\ref{app:experiments}). The latter can be reproduced with the code available at \url{https://github.com/nbariletto/MPD-DBC}.

\section{Proofs}\label{app:proofs}

In this Appendix, we report all the proofs of the theoretical results presented in the main body of the paper.

\subsection{Proof of Theorem~1}
Let $(\mathcal{F}_k)_{k \ge 0}$ be the filtration where $\mathcal{F}_0$ is the trivial $\sigma$-algebra and $\mathcal{F}_k = \sigma(Y_1, \dots, Y_k)$ for $k \ge 1$. The update rule is
$$ \theta_{n,k} = \theta_{n,k-1} + \eta_{n,k}s(Y_k;\theta_{n,k-1}). $$
Taking the conditional expectation under $\mathbb{P}$ conditional on $\mathcal{F}_{k-1}$ and applying (i), we get
$$ \mathbb{E}[\theta_{n,k} \mid \mathcal{F}_{k-1}] = \theta_{n,k-1} + \eta_{n,k}\mathbb{E}_{Y_k \sim f_{\theta_{n,k-1}}}[s(Y_k;\theta_{n,k-1})] = \theta_{n,k-1}. $$
Thus, $(\theta_{n,k})_{k \ge 0}$ is a martingale. To bound its $L^2$ norm, rewrite
\begin{align*}
 \mathbb{E}[\|\theta_{n,k}\|^2 \mid \mathcal{F}_{k-1}] & = \|\theta_{n,k-1}\|^2 + 2\eta_{n,k}\mathbb{E}[\langle \theta_{n,k-1}, s(Y_k;\theta_{n,k-1}) \rangle \mid \mathcal{F}_{k-1}] \\
 & + \eta_{n,k}^2\mathbb{E}[\|s(Y_k;\theta_{n,k-1})\|^2 \mid \mathcal{F}_{k-1}].
\end{align*}
Since $\mathbb{E}[s(Y_k;\theta_{n,k-1}) \mid \mathcal{F}_{k-1}] = 0$, the cross-term vanishes. Applying (ii),
\begin{align*}
\mathbb{E}[\|\theta_{n,k}\|^2 \mid \mathcal{F}_{k-1}] & \le \|\theta_{n,k-1}\|^2 + \eta_{n,k}^2 (A + B\|\theta_{n,k-1}\|^2) \\
& = (1 + B\eta_{n,k}^2)\|\theta_{n,k-1}\|^2 + A\eta_{n,k}^2.
\end{align*}
Taking unconditional expectations under $\mathbb{P}$ and letting $v_k := \mathbb{E}[\|\theta_{n,k}\|^2]$, with $v_0 = \|\theta_{n,0}\|^2$, we obtain
$$ v_k \le (1 + B\eta_{n,k}^2)v_{k-1} + A\eta_{n,k}^2. $$
Unrolling to $k=1$:
$$ v_k \le \|\theta_{n,0}\|^2 \prod_{j=1}^k (1 + B\eta_{n,j}^2) + A \sum_{j=1}^k \eta_{n,j}^2 \prod_{i=j+1}^k (1 + B\eta_{n,i}^2). $$
Using $1 + x \le e^x$ for $x \ge 0$:
$$ \prod_{j=1}^k (1 + B\eta_{n,j}^2) \le \exp\left(B \sum_{j=1}^k \eta_{n,j}^2\right). $$
For $\eta_{n,j} = \eta_0 / (n+j-1)$, let $S := \sum_{j=1}^\infty \eta_{n,j}^2 < \infty$. Then
$$ \sup_{k \ge 1} \mathbb{E}[\|\theta_{n,k}\|^2] \le (\|\theta_{n,0}\|^2 + A S) e^{BS} < \infty. $$
Hence, the sequence is an $L^2$-bounded martingale, converging almost surely and in $L^2(\mathbb{P})$ by Doob's Martingale Convergence Theorem.

\subsection{Proof of Theorem~2}

By (C1) and (C2), the Dominated Convergence Theorem applies, yielding
$$ \int_{\mathbb{R}^p} \nabla_\theta f_\theta(y) dy = \nabla_\theta \int_{\mathbb{R}^p} f_\theta(y) dy = 0. $$
Using the identity $\nabla_\theta f_\theta(y) = f_\theta(y) \nabla_\theta \log f_\theta(y) = f_\theta(y) s(y;\theta)$, we get $$\mathbb{E}_{Y \sim f_\theta}[s(Y;\theta)] = 0,$$ establishing (i).

To establish (ii), note that $\log p_Z(z) = -\frac{1}{2}\|z\|^2 - \frac{p}{2}\log(2\pi)$. The score function is
$$ s(y;\theta) = - \left( \partial_\theta G_\theta(y) \right)^\top G_\theta(y) + \nabla_\theta \log \left| \det \partial_y G_\theta(y) \right|. $$
Evaluating the expectation over $Y \sim f_\theta$ is equivalent to taking the expectation over $Z \sim p_Z$ at $Y = D_\theta(Z)$. Since $G_\theta(D_\theta(Z)) = Z$, substitution yields
$$ s(D_\theta(Z);\theta) = - \left( \partial_\theta G_\theta(y) \big|_{y = D_\theta(Z)} \right)^\top Z + \nabla_\theta \log \left| \det \partial_y G_\theta(y) \right| \big|_{y = D_\theta(Z)}. $$
Applying the inequality $\|a + b\|^2 \le 2\|a\|^2 + 2\|b\|^2$, the sub-multiplicative property $\|M^\top v\| \le \|M\|_F \|v\|$, and (C3):
\begin{align*}
    \|s(D_\theta(Z);\theta)\|^2 &\le 2 \left\| \left( \partial_\theta G_\theta(y) \big|_{y = D_\theta(Z)} \right)^\top Z \right\|^2 + 2 \left\| \nabla_\theta \log \left| \det \partial_y G_\theta(y) \right| \big|_{y = D_\theta(Z)} \right\|^2 \\
    &\le 2 \left\| \partial_\theta G_\theta(y) \big|_{y = D_\theta(Z)} \right\|_F^2 \|Z\|^2 + 2 \left\| \nabla_\theta \log \left| \det \partial_y G_\theta(y) \right| \big|_{y = D_\theta(Z)} \right\|^2 \\
    &\le 2 h_1(Z)^2 (1 + \|\theta\|)^2 \|Z\|^2 + 2 h_2(Z)^2 (1 + \|\theta\|)^2 \\
    &= 2 (1 + \|\theta\|)^2 \left( \|Z\|^2 h_1(Z)^2 + h_2(Z)^2 \right).
\end{align*}
Taking the expectation with respect to $Z \sim p_Z$ gives
$$ \mathbb{E}_{Y \sim f_\theta}[\|s(Y;\theta)\|^2] \le 2 (1 + \|\theta\|)^2 \mathbb{E}_{Z \sim p_Z}\left[ \|Z\|^2 h_1(Z)^2 + h_2(Z)^2 \right]. $$
By (C3), $K := \mathbb{E}_{Z \sim p_Z}\left[ \|Z\|^2 h_1(Z)^2 + h_2(Z)^2 \right] < \infty$. Using the expansion $(1 + \|\theta\|)^2 \le 2 + 2\|\theta\|^2$ yields
$$ \mathbb{E}_{Y \sim f_\theta}[\|s(Y;\theta)\|^2] \le 4K + 4K\|\theta\|^2. $$
Setting $A = 4K$ and $B = 4K$ establishes condition (ii).

\subsection{Proof of Theorem~3}

Define $\mathcal{E}_n := \{ X_{1:n} : d(f_{\theta_{n,0}}, f_*) \le \delta/2 \}$. By consistency of $f_{\theta_{n,0}}$ and $\varepsilon_n \to 0$, $\lim_{n \to \infty} F_*^n(\mathcal{E}_n^c) = 0$. Let $n$ be sufficiently large such that $\tilde\varepsilon_n \le \delta$.

We now fix $X_{1:n} \in \mathcal{E}_n$. Define the stopping time
$$ \tau := \inf \left\{ k \ge 1 : \|\theta_{n,k} - \theta_{n,0}\| > \left(\frac{\delta}{2L}\right)^{1/\gamma} \right\} $$
For $k < \tau$, Assumption (I) yields $d(f_{\theta_{n,k}}, f_{\theta_{n,0}}) \le \delta/2$. By the triangle inequality, $d(f_{\theta_{n,k}}, f_*) \le \delta$, implying $\theta_{n,k} \in \mathcal{V}_\delta(f_*)$ for $k < \tau$.

Because of the 0 expectation of the score at all $\theta\in\Theta$, $\theta_{n,k}$ is a martingale and $\tilde{\theta}_{n,k} := \theta_{n, k \wedge \tau}$ is also a (stopped) martingale. For any finite $N \ge 1$, by condition (II) we have
\begin{align*}
\mathbb{E}_{\mathbb{P}}\big[\|\tilde{\theta}_{n,N} - \theta_{n,0}\|^2\big]
&= \sum_{k=1}^N \eta_{n,k}^2 \mathbb{E}_{\mathbb{P}}\big[\|s(Y_k; \tilde\theta_{n,k-1})\|^2 1\{k \le \tau\}\big] \\
&\le M_\delta \sum_{k=1}^N \frac{\eta_0^2}{(n+k-1)^2} \\
&\le \frac{M_\delta \eta_0^2}{n-1}.
\end{align*}
Since $\sup_{N \ge 1} \mathbb{E}_{\mathbb{P}}\big[\|\tilde{\theta}_{n,N} - \theta_{n,0}\|^2\big] < \infty$, $(\tilde{\theta}_{n,N})_{N \ge 1}$ is an $L^2$-bounded martingale. By Doob's Martingale Convergence Theorem, $\tilde{\theta}_{n,\infty} := \lim_{N\to\infty}\tilde{\theta}_{n,N}$ exists $\mathbb{P}$-a.s. and in $L^2(\mathbb{P})$, satisfying:
$$ \mathbb{E}_{\mathbb{P}}\big[\|\tilde{\theta}_{n,\infty} - \theta_{n,0}\|^2\big] \le \frac{M_\delta \eta_0^2}{n-1}. $$

By the union bound:
\begin{align*}
\Pi_n\Big( d(f_\theta, f_{\theta_{n,0}}) > \frac{\tilde\varepsilon_n}{2} \Big) 
&\le \mathbb{P}(\tau < \infty) + \mathbb{P}\left( d(f_{\tilde\theta_{n,\infty}}, f_{\theta_{n,0}}) > \frac{\tilde\varepsilon_n}{2} \right).
\end{align*}

Applying Doob's $L^2$ maximal inequality to the submartingale $\|\tilde{\theta}_{n,k} - \theta_{n,0}\|$ over finite time $N$:
\begin{align*}
\mathbb{P}\left( \sup_{1 \le k \le N} \|\tilde{\theta}_{n,k} - \theta_{n,0}\| > \left(\frac{\delta}{2L}\right)^{1/\gamma} \right)
&\le \frac{\mathbb{E}_{\mathbb{P}}[\|\tilde{\theta}_{n,N} - \theta_{n,0}\|^2]}{(\delta / 2L)^{2/\gamma}} \\
&\le \frac{M_\delta \eta_0^2}{(n-1)(\delta / 2L)^{2/\gamma}}.
\end{align*}
Taking $N \to \infty$:
\begin{align*}
\mathbb{P}(\tau < \infty) &= \mathbb{P}\left( \sup_{k \ge 1} \|\tilde{\theta}_{n,k} - \theta_{n,0}\| > \left(\frac{\delta}{2L}\right)^{1/\gamma} \right) \\
&\le \frac{M_\delta \eta_0^2}{(n-1)(\delta / 2L)^{2/\gamma}}.
\end{align*}

Applying Assumption (I) and Markov's inequality to the limit:
\begin{align*}
\mathbb{P}\left( d(f_{\tilde\theta_{n,\infty}}, f_{\theta_{n,0}}) > \frac{\tilde\varepsilon_n}{2} \right) 
&\le \mathbb{P}\left( \|\tilde\theta_{n,\infty} - \theta_{n,0}\| > \left(\frac{\tilde\varepsilon_n}{2L}\right)^{1/\gamma} \right) \\
&\le \frac{\mathbb{E}_{\mathbb{P}}[\|\tilde{\theta}_{n,\infty} - \theta_{n,0}\|^2]}{(\tilde\varepsilon_n / 2L)^{2/\gamma}} \\
&\le \frac{M_\delta \eta_0^2}{(n-1)(\tilde\varepsilon_n / 2L)^{2/\gamma}}.
\end{align*}

Let $B_n(\tilde\varepsilon_n)$ be the sum of the upper bounds for $\mathbb{P}(\tau < \infty)$ and $\mathbb{P}(d(f_{\tilde\theta_{n,\infty}}, f_{\theta_{n,0}}) > \tilde\varepsilon_n/2)$. By the triangle inequality, $d(f_\theta, f_*) > \tilde\varepsilon_n$ implies either $d(f_\theta, f_{\theta_{n,0}}) > \tilde\varepsilon_n/2$ or $d(f_{\theta_{n,0}}, f_*) > \tilde\varepsilon_n/2$. Thus:
\begin{align*}
\Pi_n\big( d(f_\theta, f_*) > \tilde\varepsilon_n \big) 
&\le \Pi_n\Big( d(f_\theta, f_{\theta_{n,0}}) > \frac{\tilde\varepsilon_n}{2} \Big) + 1\left\{d(f_{\theta_{n,0}}, f_*) > \frac{\tilde\varepsilon_n}{2}\right\}.
\end{align*}

Taking the expectation with respect to $F_*^n$:
\begin{align*}
\mathbb{E}_{F_*^n}\big[ \Pi_n\big( d(f_\theta, f_*) > \tilde\varepsilon_n \big) \big] 
&\le \mathbb{E}_{F_*^n}\left[ \Pi_n\Big( d(f_\theta, f_{\theta_{n,0}}) > \frac{\tilde\varepsilon_n}{2} \Big) \mathbf{1}_{\mathcal{E}_n} \right] \\
& + \mathbb{E}_{F_*^n}\left[ \Pi_n\Big( d(f_\theta, f_{\theta_{n,0}}) > \frac{\tilde\varepsilon_n}{2} \Big) \mathbf{1}_{\mathcal{E}_n^c} \right] \\
& + F_*^n\left(d(f_{\theta_{n,0}}, f_*) > \frac{\tilde\varepsilon_n}{2}\right) \\
&\le B_n(\tilde\varepsilon_n) + F_*^n(\mathcal{E}_n^c) + F_*^n\left(d(f_{\theta_{n,0}}, f_*) > \frac{\tilde\varepsilon_n}{2}\right).
\end{align*}

Since $\lim_{n \to \infty} n\tilde\varepsilon_n^{2/\gamma} = \infty$, then $\lim_{n \to \infty} B_n(\tilde\varepsilon_n) = 0$. Given $\tilde\varepsilon_n \gtrsim \varepsilon_n$ and the $(d, \varepsilon_n)$-consistency of $f_{\theta_{n,0}}$, then $\lim_{n \to \infty} F_*^n(\mathcal{E}_n^c) = 0$ and $\lim_{n \to \infty} F_*^n(d(f_{\theta_{n,0}}, f_*) > \tilde\varepsilon_n/2) = 0$. Therefore
$$ \lim_{n \to \infty} \mathbb{E}_{F_*^n}\big[ \Pi_n\big( d(f_\theta, f_*) > \tilde\varepsilon_n \big) \big] = 0. $$
Convergence in $L^1(F_*^n)$ implies convergence in $F_*^n$-probability, yielding the result.

\subsection{Proof of Theorem~4}
Following the proof of Theorem 3 in \cite{cuevas2006plug}, we have that
    \begin{equation*}
        \lambda(\mathcal{L}(f)\triangle \mathcal{L}(f_*)) \leq \lambda(\{x\in\mathcal X : |f_*(x)-c|\leq \varepsilon\}) + \lambda(\{x\in\mathcal X : |f_*(x)-f(x)|> \varepsilon\})
    \end{equation*}
    for any $\varepsilon>0$. Therefore, for any $n\in\mathbb{N}$ and $\delta>0$,
    \begin{align*}
        & \Pi_n (\lambda(\mathcal{L}(f)\triangle \mathcal{L}(f_*))>\delta) \\
        & \leq 1\{\lambda(\{x\in\mathcal X : |f_*(x)-c|\leq \varepsilon\})>\delta/2\} \\
        & + \Pi_n(\lambda(\{x\in\mathcal X : |f_*(x)-f(x)|> \varepsilon\})>\delta/2)\\
        & \leq1\{\lambda(\{x\in\mathcal X : |f_*(x)-c|\leq \varepsilon\})>\delta/2\} + \Pi_n(d(f,f_*)>\varepsilon).
    \end{align*}
    Now replace $\varepsilon=\tilde\varepsilon_n$ and $\delta=2C_0\tilde\varepsilon_n$. For $n$ large enough so that $\tilde\varepsilon_m\leq \varepsilon_0$ for all $n\geq m$, we obtain
    \begin{equation*}
        \Pi_n(\lambda(\mathcal{L}(f)\triangle \mathcal{L}(f_*))>2C_0\tilde\varepsilon_n) \leq \Pi_n(d(f,f_*)>\tilde\varepsilon_n).
    \end{equation*}
    Applying the definition of MPD contraction rate in the statement of Theorem~3 completes the proof.

\subsection{Proof of Theorem~5}

Define the event $E_n = \{ f : \|f - f_*\|_{\mathcal L^\infty} \leq \tilde{\varepsilon}_n \}$ and the sequence $\eta_n:=4\tilde\varepsilon_n$. By the posterior contraction assumption, $\lim_{n\to\infty}\Pi_n(E_n^c) = 0$ in $F_*^n$-probability. Hence, it suffices to show that, for sufficiently large $n$, $f \in E_n$ implies that $k_f^{\eta_n} = k_{f_*}$ and $\inf_{\sigma \in S_{k_{f_*}}} \sum_{j=1}^{k_{f_*}} \lambda\big(C_j(f_*) \triangle C_{\sigma(j)}\big) \leq k_{f_*} C_0\tilde\varepsilon_n$, where $C_1,\ldots,C_{k_{f_*}}$ denotes an enumeration of $\mathcal{C}_c^{\eta_n}(f)$. To that end, fix $n$ large enough such that $\tilde{\varepsilon}_m < \min(\delta_0,\eta_0,\varepsilon_0)/8$ for all $m\ge n$, and take $f \in E_n$.

We first verify that there is a one-to-one correspondence between $\mathcal{C}(f_*)$ and a subcollection of $\mathcal{C}(f)$. Suppose a cluster $C \in \mathcal{C}(f)$ intersects two distinct true clusters $C_i(f_*)$ and $C_j(f_*)$. Since $C$ is path-connected, there exists a path $\gamma$ contained in it and connecting the two sets. For all $x$ in the range of $\gamma$, $f(x) \geq c$, which implies $f_*(x) \geq f(x) - \tilde{\varepsilon}_n \geq c - \tilde{\varepsilon}_n>c-\delta_0/8$. This contradicts condition (b), so any cluster of $f$ intersects at most one cluster of $f_*$. Next, let $j\in\{1,\ldots,k_{f_*}\}$ and recall the core $K_j^{\tilde{\varepsilon}_n} = \{x \in C_j(f_*) : f_*(x) \geq c + \tilde{\varepsilon}_n\}$. By condition (c) and because $\tilde\varepsilon_n<\eta_0$, this set is path-connected. For any $x \in K_j^{\tilde{\varepsilon}_n}$, we have $f(x) \geq f_*(x) - \tilde{\varepsilon}_n \geq c$, so $K_j^{\tilde{\varepsilon}_n} \subseteq \mathcal{L}(f)$. Since it is path-connected, $K_j^{\tilde{\varepsilon}_n}$ must be contained in a single component of $\mathcal{L}(f)$, which we denote $C_{\sigma(j)}(f)$. By our previous finding, this component intersects no true cluster other than $C_j(f_*)$, implying that $\sigma$ is an injective map from $\{1,\ldots,k_{f_*}\}$ into the index set of $\mathcal{C}(f)$.

Fix $j\in\{1,\ldots,k_{f_*}\}$. Since condition (c) guarantees that $K_j^\eta\subseteq C_j(f_*)$ is nonempty for every $\eta\in(0,\eta_0)$, we have $M_j^*:=\sup_{x\in C_j(f_*)}f_*(x)= \sup_{x\in K_j^\eta}f_*(x) \geq c+\eta$ for all $\eta\in(0,\eta_0)$; letting $\eta\to\eta_0$ shows that $M_j^*\ge c+\eta_0$. Since $K_j^{\tilde\varepsilon_n}\subseteq C_{\sigma(j)}(f)$ and $\tilde\varepsilon_n<\eta_0$, we have
\begin{equation}\label{eq:Mj_lower}
M(C_{\sigma(j)}(f),f)\ge\sup_{x\in K_j^{\tilde\varepsilon_n}}f(x)\ge \sup_{x\in K_j^{\tilde\varepsilon_n}}f_*(x) - \tilde\varepsilon_n = M_j^*-\tilde\varepsilon_n\ge c+\eta_0-\tilde\varepsilon_n.
\end{equation}
Now let $C\in\mathcal{C}(f)\setminus\{C_{\sigma(1)}(f),\ldots,C_{\sigma(k_{f_*})}(f)\}$. Every $x\in C$ satisfies $f(x)\ge c$, hence $f_*(x)\ge c-\tilde\varepsilon_n>c-\eta_0$; by condition (d), this implies $x\in\bigcup_{\ell=1}^{k_{f_*}}N_\ell$, so $x\in N_{\ell(x)}$ for some $\ell(x)$. {Since each $N_\ell$ is closed (as $x\mapsto\operatorname{dist}(x,C_\ell(f_*))$ is continuous) and the sets $N_1,\ldots,N_{k_{f_*}}$ are pairwise disjoint by condition (d), the collection $\{N_\ell\cap C:\ell=1,\ldots,k_{f_*}\}$ is a finite partition of $C$ into relatively closed subsets. Moreover, each $N_\ell\cap C$ is also relatively open in $C$, as its complement within $C$ is the finite union $\bigcup_{m\ne\ell}(N_m\cap C)$ of relatively closed sets. As $C$ is path-connected, hence connected, it cannot be partitioned into two or more nonempty relatively open subsets, so exactly one of the sets $N_\ell\cap C$ is nonempty and equals $C$; denote the corresponding index by $\ell_0$, so that $\ell(x)=\ell_0$ for every $x\in C$.}

For any $x\in C$, if $x\in C_\ell(f_*)$ for some $\ell$, then $x\in N_\ell$ (as trivially $\operatorname{dist}(x,C_\ell(f_*))=0\le\rho_0$), which forces $\ell=\ell_0$ by disjointness; hence, for $x\in C$, either $f_*(x)<c$ (if $x\notin\bigcup_{\ell=1}^{k_{f_*}} C_\ell(f_*)\equiv \mathcal L(f_*)$) or $x\in C_{\ell_0}(f_*)$. In the latter case, $x\notin K_{\ell_0}^{\eta_0/2}$: since $\eta_0/2<\eta_0$ implies $K_{\ell_0}^{\eta_0/2}\subseteq K_{\ell_0}^{\tilde\varepsilon_n}\subseteq C_{\sigma(\ell_0)}(f)$, having $x\in K_{\ell_0}^{\eta_0/2}$ would imply $x\in C_{\sigma(\ell_0)}(f)$, forcing $C=C_{\sigma(\ell_0)}(f)$ since the path-connected components of $\mathcal{L}(f)$ are disjoint, and ultimately contradicting the fact that $C\in\mathcal{C}(f)\setminus\{C_{\sigma(1)}(f),\ldots,C_{\sigma(k_{f_*})}(f)\}$. Combining both cases, we deduce that $f_*(x)<\max\{c,c+\eta_0/2\}=c+\eta_0/2$ for all $x\in C$, and so
\begin{equation}\label{eq:MC_upper}
M(C,f) = \sup_{x\in C} f(x) \le \sup_{x\in C} f_*(x) + \tilde\varepsilon_n \leq c+\eta_0/2+\tilde\varepsilon_n.
\end{equation}

Let $C$ and $\ell_0$ be as above (so every $x\in C$ satisfies $x\in N_{\ell_0}$ and $f_*(x)\ge f(x)-\tilde\varepsilon_n\ge c-\tilde\varepsilon_n$), and pick any $x'\in C$. Since $\tilde\varepsilon_n<\eta_0/8$, $f_*(x')\ge c-\tilde\varepsilon_n>c-\eta_0$, so condition (d) applies: there is a continuous path $\gamma:[0,1]\to N_{\ell_0}$ with $\gamma(0)=x'$, $\gamma(1)\in K_{\ell_0}^{\eta_0/2}$, and $f_*(\gamma(t))$ non-decreasing in $t$. Hence $f_*(\gamma(t))\ge f_*(x')\ge c-\tilde\varepsilon_n$ for all $t\in[0,1]$, so $f(\gamma(t))\ge f_*(\gamma(t))-\tilde\varepsilon_n\ge c-2\tilde\varepsilon_n>c-\eta_n$, i.e., $\gamma(t)\in\mathcal L_{c-\eta_n}(f)$ for all $t\in[0,1]$. So $\gamma$ connects $x'\in C$ to $\gamma(1)\in K_{\ell_0}^{\eta_0/2}\subseteq K_{\ell_0}^{\tilde\varepsilon_n}\subseteq C_{\sigma(\ell_0)}(f)\subseteq\mathcal L_{c-\eta_n}(f)$. Since $C\subseteq\mathcal L(f)\subseteq\mathcal L_{c-\eta_n}(f)$ is itself path-connected, $C$ and $C_{\sigma(\ell_0)}(f)$ lie in the same path-connected component of $\mathcal L_{c-\eta_n}(f)$. By \eqref{eq:Mj_lower}, \eqref{eq:MC_upper}, and $\tilde\varepsilon_n<\eta_0/8$,
\[
M(C_{\sigma(\ell_0)}(f),f) \ge c+\eta_0-\tilde\varepsilon_n > c+\eta_0/2+\tilde\varepsilon_n \ge M(C,f),
\]
so $C$ is not $\eta_n$-persistent. This implies that $\mathcal C^{\eta_n}(f)\subseteq\{C_{\sigma(1)}(f),\ldots,C_{\sigma(k_{f_*})}(f)\}$ and therefore $k_f^{\eta_n}\leq k_{f_*}$

We now show that each cluster $C_{\sigma(j)}(f)$ is $\eta_n$-persistent. From the previous paragraph, we know that for any $j=1,\ldots,k_{f_*}$ and  $C\in\mathcal{C}(f)\setminus\{C_{\sigma(1)}(f),\ldots,C_{\sigma(k_{f_*})}(f)\}$, if $C$ and $C_{\sigma(j)}(f)$ belong to the same path-connected component of $\mathcal L_{c-\eta_n}(f)$, then $M(C_{\sigma(j)}(f),f)>M(C,f)$. Hence, $\eta_n$-persistence of $C_{\sigma(j)}(f)$ may be violated only if $C_{\sigma(j)}(f)$ is contained in the same path-connected component of $\mathcal L_{c-\eta_n}(f)$ as $C_{\sigma(\ell)}(f)$ for some $\ell\ne j$. If this were the case, there would exist a continuous path $\gamma$ from $\gamma(0)\in K_j^{\tilde\varepsilon_n}\subseteq C_{\sigma(j)}(f)$ to $\gamma(1)\in K_\ell^{\tilde\varepsilon_n}\subseteq C_{\sigma(\ell)}(f)$, with $\gamma(t)\in\mathcal{L}_{c-\eta_n}(f)$ for all $t\in[0,1]$, such that
$$\forall t\in[0,1],\qquad f_*(\gamma(t))\ge f(\gamma(t))-\tilde\varepsilon_n\ge c-\eta_n-\tilde\varepsilon_n = c-5\tilde\varepsilon_n>c-5\delta_0/8,$$
where we used the fact that $\tilde\varepsilon_n<\delta_0/8$. Because $\gamma(0)\in C_j(f_*)$ and $\gamma(1)\in C_\ell(f_*)$ with $j\ne\ell$, this would contradict condition (b). This implies that $C_{\sigma(j)}(f)$ is $\eta_n$-persistent for every $j=1,\ldots,k_{f_*}$, so that $\{C_{\sigma(1)}(f),\ldots,C_{\sigma(k_{f_*})}(f)\}\subseteq\mathcal C^{\eta_n}(f)$ and $k_f^{\eta_n}\geq k_{f_*}$. Combining this with the conclusion from the previous paragraph, we obtain that $k_f^{\eta_n}=k_{f_*}$ for every $f\in E_n$ and $n$ large enough, which establishes Equation~(3) in the main text.

Now consider the symmetric difference for each pair $(C_j(f_*), C_{\sigma(j)}(f))$. If $x \in C_{\sigma(j)}(f) \setminus C_j(f_*)$, then $f(x) \geq c$ and $f_*(x) < c$, implying $c > f_*(x) \geq f(x) - \tilde{\varepsilon}_n \geq c - \tilde{\varepsilon}_n$. If $x \in C_j(f_*) \setminus C_{\sigma(j)}(f)$, then $x$ must lie outside the core $K_j^{\tilde{\varepsilon}_n}$, so $c \leq f_*(x) < c + \tilde{\varepsilon}_n$. This shows that $C_{\sigma(j)}(f) \triangle C_j(f_*) \subseteq \{x : |f_*(x) - c| \leq \tilde{\varepsilon}_n\}$. Using condition (a),
\begin{equation*}
    \lambda(C_{\sigma(j)}(f) \triangle C_j(f_*)) \leq \lambda(\{x : |f_*(x) - c| \leq \tilde{\varepsilon}_n\}) \leq C_0 \tilde{\varepsilon}_n.
\end{equation*}
Summing over all clusters, we obtain
\begin{equation*}
    \sum_{j=1}^{k_{f_*}} \lambda\big(C_j(f_*) \triangle C_{\sigma(j)}(f)\big) \leq k_{f_*} C_0 \tilde{\varepsilon}_n.
\end{equation*}
Together with the above-established fact that $\mathcal C^{\eta_n}(f)=\{C_{\sigma(1)}(f),\ldots,C_{\sigma(k_{f_*})}(f)\}$, this yields Equation~(4) in the main text and completes the proof of Theorem~5.

\section{Experiment details}\label{app:experiments}

In this Appendix, we provide full details on the numerical experiments presented in the main body of the article. The code to reproduce them is available at \url{https://github.com/nbariletto/MPD-DBC}.

\subsection{Noisy concentric circles}

\paragraph{Dataset.} We generate $n = 5{,}000$ two-dimensional observations from two noisy concentric circles using a fixed random seed, with Gaussian noise of standard deviation $0.15$ and a radius factor of $0.25$. The outer and inner circles contain approximately $4{,}000$ and $1{,}000$ points respectively, reflecting the ratio imposed by the factor parameter. Data are standardized to zero mean and unit variance before fitting.

\paragraph{Density Estimator.} We use a Masked Autoregressive Flow \citep{papamakarios2017masked} implemented via the \texttt{surjectors} library \citep{dirmeier2024surjectors}. The architecture consists of 12 alternating MAF layers, each with a MADE conditioner \citep{germain2015made} comprising two hidden layers of width 128, interleaved with reverse permutation layers. The base distribution is a standard two-dimensional Gaussian. The model is trained by maximum likelihood using AdamW \citep{loshchilov2019decoupled} with gradient clipping (global norm $\leq 1$), weight decay $10^{-4}$, and a warmup-cosine decay learning rate schedule peaking at $10^{-3}$ and decaying to $10^{-6}$ over $10{,}000$ epochs with batch size $5{,}000$.

\paragraph{Predictive Resampling.} We generate $T = 1{,}000$ independent resamples, each running for $N = 3{,}000$ score-based update steps with base learning rate $\eta_0 = 0.02$ and schedule $\eta_k = \eta_0 / (n + k-1)$. Gradients are clipped to $[-100, 100]$ and NaN values are replaced with zero for numerical stability. The $T$ resampling chains are executed in parallel on a single GPU.

\paragraph{Clustering Algorithm.} The clustering procedure applied to each (trained or resampled) density is described in Algorithm~\ref{alg:dbc}. The density threshold $\tau$ is set to the 10th percentile of the log-density values of the training data under the initially trained model $f_{\theta_{n,0}}$, and is held fixed across all resamples. The neighborhood radius $r$ is computed adaptively from the training data as $1.2$ times the mean 10th nearest-neighbor distance among points above the threshold, and is likewise fixed across resamples. The minimum cluster size is set to $m = 100$.

\begin{algorithm}[t]
\caption{Density-Based Clustering via Upper-Level Sets}
\label{alg:dbc}
\begin{algorithmic}[1]
\Require Data $X = \{x_1, \ldots, x_n\}$, log-density values $\{\log f(x_i)\}_{i=1}^n$, threshold $\tau$, radius $r$, minimum cluster size $m$
\Ensure Cluster labels $\ell_1, \ldots, \ell_n$

\State Initialize $\ell_i \leftarrow -1$ for all $i$
\State Identify core points: $\mathcal{D} \leftarrow \{i : \log f(x_i) > \tau\}$
\State Build radius neighborhood graph $G$ on $\{x_i : i \in \mathcal{D}\}$ with radius $r$
\State Compute connected components of $G$; assign component labels to $\{\ell_i : i \in \mathcal{D}\}$
\State Let $\mathcal{K}_{\text{large}}$ be the set of component labels with $\geq m$ members
\State Let $\mathcal{K}_{\text{small}}$ be the set of component labels with $< m$ members
\ForAll{$i$ with $\ell_i \in \mathcal{K}_{\text{small}}$}
    \State $\ell_i \leftarrow$ label of the nearest point $j$ with $\ell_j \in \mathcal{K}_{\text{large}}$
\EndFor
\ForAll{$i$ with $\ell_i = -1$}
    \State $\ell_i \leftarrow$ label of the nearest point $j$ with $\ell_j \neq -1$
\EndFor
\State Re-index labels to $\{0, 1, \ldots, K-1\}$
\State \Return $\ell_1, \ldots, \ell_n$
\end{algorithmic}
\end{algorithm}



\subsection{MNIST digits}

\paragraph{Dataset.} We consider a training subset of $n = 5{,}000$ images from the MNIST dataset \citep{lecun1998gradient} containing a balanced number of digits 3 and 8, alongside an out-of-sample (OOS) evaluation set of $1{,}000$ unseen images (500 per class). Images are $28 \times 28$ grayscale pixels, normalized to $[0,1]$.

\paragraph{Dimensionality Reduction.} Prior to density estimation, images are embedded into a $24$-dimensional latent space using a convolutional autoencoder trained on the $5{,}000$ training images with mean squared error reconstruction loss. The encoder consists of three convolutional layers (32, 64, and 128 channels, with stride-2 convolutions and ReLU activations) followed by a linear projection to the latent space; the decoder mirrors this architecture with transposed convolutions. The autoencoder is trained for 25 epochs using Adam \citep{kingma2015adam} with learning rate $10^{-3}$ and batch size 128. Latent codes for both the training and OOS sets are extracted and independently standardized to zero mean and unit variance before proceeding.

\paragraph{Density Estimator.} We fit a MAF \citep{papamakarios2017masked} to the standardized training latent codes. The architecture consists of 16 alternating MAF layers, each with a MADE conditioner \citep{germain2015made} comprising two hidden layers of width 128, interleaved with reverse permutation layers. The base distribution is a standard 24-dimensional Gaussian. The model is trained by maximum likelihood on the training set using Adam \citep{kingma2015adam} with learning rate $10^{-4}$ and batch size 128 for 15 epochs.

\paragraph{Predictive Resampling.} Using the training data, we generate $T = 500$ independent resamples, each running for $N = 3{,}000$ score-based update steps with base learning rate $\eta_0 = 0.002$ and schedule $\eta_k = \eta_0 / (n + k-1)$. The $T$ resampling chains are executed in parallel on a single GPU. The resulting $T$ resampled densities are subsequently evaluated on the combined set of $6{,}000$ training and OOS latent codes.

\paragraph{Clustering.} For each resampled density, clustering is performed jointly on the combined latent codes using the ToMATo algorithm \citep{chazal2013persistence} with a $k$-nearest-neighbor graph ($k = 20$, brute-force) as the underlying neighborhood structure and a fixed merge threshold of $0.5$ (visually picked from the ToMATo persistence diagram corresponding to the trained baseline density; the code to reproduce the results produces such plots). Density weights passed to ToMATo are obtained by exponentiating and normalizing the log-density values.

\paragraph{Conformal Credible Set.} For uncertainty quantification, we apply the Conformalized Bayesian Inference methodology of \cite{bariletto2025conformalized}, using the Python library \texttt{cbi\_partitions} to process the sampled partitions of the OOS subset. The $T = 500$ posterior clustering samples (restricted to the $1{,}000$ OOS digits) are split evenly into training ($250$) and calibration ($250$) sets. A partition pseudo-KDE score is fitted with the Variation-of-Information loss metric \citep{meila2007comparing} and bandwidth parameter $0.5$. The resulting conformal $p$-value for the true OOS digit labeling is $0.2151$, confirming that the true labeling belongs to a credible set with guaranteed $90\%$ marginal coverage under the MPD.

\subsection{Bone marrow scRNA-seq data}

\paragraph{Dataset.} We applied our method to the Tabula Sapiens Bone Marrow dataset \citep{tabulasapiens2022,quake2025tabula}, a single-cell RNA sequencing (scRNA-seq) atlas comprising $n = 27{,}112$ individual cells annotated into 26 distinct cell types. The dataset, accessed via the CZ CELLxGENE portal,\footnote{At \url{https://datasets.cellxgene.cziscience.com/c7f0c3ea-2083-4d87-a8e0-7f69626aa40d.h5ad}.} includes pre-computed Principal Component Analysis (PCA) and Uniform Manifold Approximation and Projection (UMAP) coordinates.

\paragraph{Dimensionality Reduction.} As is standard practice in scRNA-seq analyses, we utilized the first $p = 10$ principal components as the continuous feature representations (latent codes) for density estimation and clustering. Prior to fitting the density model, these 10-dimensional PCA features were standardized to zero mean and unit variance. The 2-dimensional UMAP coordinates were reserved exclusively for visualization purposes.

\paragraph{Density Estimator.} We modeled the underlying continuous distribution of the 10-dimensional PCA features using a MAF \citep{papamakarios2017masked}. The normalizing flow architecture consisted of 16 autoregressive layers, each parameterized by a MADE \citep{germain2015made} containing two hidden layers of 128 units each. Permutation layers were interleaved between the autoregressive layers to reverse the variable ordering. The base distribution was a standard 10-dimensional Gaussian. The MAF was trained via maximum likelihood on the entire dataset using the AdamW optimizer \citep{loshchilov2019decoupled} with a batch size of 500 and a weight decay of $10^{-4}$. Training proceeded for 100 epochs, employing a cosine decay learning rate schedule starting from $10^{-4}$ and decaying to $10^{-5}$ over the total number of training steps.

\paragraph{Predictive Resampling.} Following training, we generated $T = 500$ independent parameter configurations by score-based predictive resampling. Each of the $T$ chains was run for $N = 3{,}000$ score-based update steps. We utilized a base learning rate of $\eta_0 = 5 \times 10^{-3}$, following the schedule $\eta_k = \eta_0 / (n + k-1)$ where $n = 27{,}112$. The $500$ chains were parallelized on a single GPU.

\paragraph{Clustering.} Clustering was performed on the standardized PCA features using the ToMATo algorithm \citep{chazal2013persistence}. The underlying topological structure was defined by a $k$-nearest-neighbor graph ($k = 30$, brute-force). For the initial baseline clustering and each of the $T = 500$ resampled densities, the corresponding density estimates (exponentiated and normalized by subtracting the maximum log-probability for numerical stability) were used as vertex weights. A fixed merge threshold of $0.3$ was applied across all ToMATo runs, determined by visual inspection of the persistence diagram associated with the baseline density estimate. This generated a posterior ensemble of $T = 500$ distinct cell partitions.

\paragraph{Co-clustering Matrix Computation.} To summarize the clustering uncertainty at the cell-type level, we aggregated the $T = 500$ individual cell partitions into a cell-type co-clustering probability matrix. Let $R = 500$ be the total number of resampling runs, $C$ be the number of unique cell types (26), and $N_u$ be the total number of cells of type $u$ in the dataset. For a single cluster $k$ in a given run $r$, the number of co-clustered pairs between cells of type $u$ and type $v$ is simply $n_{u,k}^{(r)} \times n_{v,k}^{(r)}$, where $n_{u,k}^{(r)}$ is the count of type $u$ cells assigned to cluster $k$. Summing these pairs across all clusters and all runs yields the total observed co-clustering events between the two types. To convert this into a probability, we normalized by the total theoretically possible number of cross-type pairings over the experiment, which is $R \times N_u \times N_v$. Then, the entries of the cell-type co-clustering probability matrix $M$ plotted in the main text are computed as:
$$M_{u,v} = \frac{1}{R \times N_u\times N_v} \sum_{r=1}^{R} \sum_{k} n_{u,k}^{(r)} n_{v,k}^{(r)}.$$

For the neutrophil-restricted co-clustering analysis, a standard cell-by-cell pairwise similarity matrix was computed over the 8,677 neutrophil cells, averaging the binary adjacency matrices (1 if in the same cluster, 0 otherwise) across the $T = 500$ runs. In both cases, rows and columns were ordered using hierarchical agglomerative clustering with average linkage based on the distance metric $1 - M$ to group highly co-clustering blocks along the diagonal.

We conclude this experiment description with the following remark. While we have tried to ensure full reproducibility through fixed random seeds, we have observed that repeated runs may occasionally yield slight numerical differences, likely due to nondeterministic GPU operations. These variations do not affect the overall interpretation of the results, and in the paper we report the outcomes most consistently observed across runs.

\bibliographystyle{agsm}

\bibliography{arxiv.bib}

\end{document}